\documentclass{article}

\usepackage[preprint]{neurips_2021}

\usepackage[utf8]{inputenc} 
\usepackage[T1]{fontenc}    
\usepackage{hyperref}       
\usepackage{url}            
\usepackage{booktabs}       
\usepackage{amsfonts}       
\usepackage{nicefrac}       
\usepackage{microtype}      
\usepackage{xcolor}         
\usepackage{amssymb}

\usepackage{times}
\usepackage{soul}
\usepackage{graphicx}
\usepackage{amsmath}
\usepackage{multirow}
\usepackage{epsfig,epstopdf}
\usepackage{enumitem}
\usepackage[most]{tcolorbox}

\usepackage{url}
\usepackage[utf8]{inputenc}
\usepackage{amsthm}
\usepackage{booktabs}
\usepackage{algorithm}
\usepackage{algorithmic}
\usepackage{multirow}
\usepackage{hhline}
\usepackage{epstopdf}
\usepackage{color}
\usepackage{enumitem,picinpar}
\usepackage{newfloat}
\usepackage{listings}
\usepackage{textcomp,moresize}
\usepackage{bbm}
\usepackage{latexsym}
\usepackage{epsfig}
\usepackage[flushleft]{threeparttable}
\usepackage{amsfonts}
\usepackage{pifont}

\usepackage{textcomp,wrapfig}
\usepackage[tableposition=top]{caption}
\usepackage{subcaption}
\usepackage[framemethod=tikz]{mdframed}

\newtheorem{theorem}{Theorem}
\newtheorem{lemma}{Lemma}
\newtheorem{definition}{Definition}
\newtheorem{assumption}{Assumption}
\newtheorem{property}{Property}
\newtheorem{remark}{Remark}
\newtheorem{corollary}{Corollary}
\newtheorem{proof*}{Proof}

\DeclareMathOperator*{\argmin}{arg\,min}
\DeclareMathOperator*{\prox}{prox}
\DeclareMathOperator{\Gap}{Gap}
\DeclareMathOperator{\supp}{supp}

\def\wrt{\textit{w.r.t.~}}

\newcommand{\E}{\mathbb{E}}

\newcommand{\PP}{\mathcal{P}}
\newcommand{\Prob}{\mathbb{P}}

\newcommand{\F}{\mathcal{F}}
\newcommand{\I}{\mathcal{I}}
\newcommand{\G}{\mathcal{G}}
\newcommand{\D}{\mathcal{D}}
\newcommand{\B}{\mathcal{S}}

\newcommand{\cmark}{\ding{51}}%
\newcommand{\xmark}{\ding{55}}%

\title{An Accelerated Doubly Stochastic Gradient Method with
Faster Explicit Model Identification}

\author{%
 Runxue Bao \\
University of Pittsburgh\\
	\texttt{runxue.bao@pitt.edu} \\
	\And Bin Gu \\
MBZUAI\\
	\texttt{jsgubin@gmail.com} 	\\
	\And Heng Huang \\
University of Pittsburgh \\
	\texttt{heng.huang@pitt.edu}
}

\begin{document}

\maketitle

\begin{abstract}
Sparsity regularized loss minimization problems play an important role in various fields including machine learning, data mining, and modern statistics. Proximal gradient descent method and coordinate descent method are the most popular approaches to solving the minimization problem. Although existing methods can achieve implicit model identification, aka support set identification, in a finite number of iterations, these methods still suffer from huge computational costs and memory burdens in high-dimensional scenarios. The reason is that the support set identification in these methods is implicit and thus cannot explicitly identify the low-complexity structure in practice, namely, they cannot discard useless coefficients of the associated features to achieve algorithmic acceleration via dimension reduction. To address this challenge, we propose a novel accelerated doubly stochastic gradient descent (ADSGD) method for sparsity regularized loss minimization problems, which can reduce the number of block iterations by eliminating inactive coefficients during the optimization process and eventually achieve faster explicit model identification and improve the algorithm efficiency. Theoretically, we first prove that ADSGD  can achieve a linear convergence rate and lower overall computational complexity. More importantly, we prove that ADSGD can achieve a linear rate of explicit model identification. Numerically, experimental results on benchmark datasets confirm the efficiency of our proposed method. 
\end{abstract}

\section{Introduction}
Learning sparse representations plays a key role in machine learning, data mining, and modern statistics \cite{lustig2008compressed,shevade2003simple,wright2010sparse,chen2021learnable,bian2021optimization, jia2021psgr}. Many popular statistical learning models, such as Lasso \cite{tibshirani1996regression}, group Lasso \cite{yuan2006model}, sparse logistic regression \cite{ng2004feature}, Sparse-Group Lasso \cite{simon2013sparse}, elastic net \cite{zou2005regularization}, sparse Support Vector Machine (SVM) \cite{zhu20031}, \emph{etc}, have been developed and achieved great success for both regression and classification tasks. Given design matrix $A \in \mathbb{R}^{n \times d}$ with $n$ observations and $d$ features, these models can be formulated as regularized loss minimization problems:
\begin{eqnarray}  \label{eq:general}
\min\limits_{x \in \mathbb{R}^{d} } \PP(x):=  \F(x) +\lambda \Omega(x),
\end{eqnarray}
where $\F(x)=\frac{1}{n}\sum_{i=1}^{n}f_i(a_i^\top x)$ is the data-fitting loss, $\Omega(x)$ is the block-separable regularizer that encourages the property of the model parameters, $\lambda$ is the regularization parameter, and $x$ is the model parameter. Let $\G$ be a partition of the coefficients, we have $\Omega(x) =  \sum_{j=1}^{q} \Omega_j(x_{\G_j}) $.

Proximal gradient descent (PGD) method  was proposed in \cite{lions1979splitting,combettes2005signal} to solve Problem (\ref{eq:general}). However, at each iteration, PGD requires gradient evaluation of all the samples, which is  computationally expensive. To address this issue, stochastic proximal gradient (SPG) method is proposed in  \cite{duchi2009efficient}, which only relies on the gradient of a sample at each iteration. However, SPG only achieves a sublinear convergence rate due to the gradient variance introduced by random sampling. Further, proximal stochastic variance-reduced gradient (ProxSVRG)  method \cite{xiao2014proximal} is proposed to achieve a linear convergence rate for strongly convex functions by adopting a variance-reduced technique, which is widely used in  \cite{defazio2014saga, li2022local}. On the other hand, coordinate descent method has received increasing attention due to its efficiency. Randomized block coordinate descent (RBCD) method is proposed in \cite{shalev2011stochastic,richtarik2014iteration}, which only updates a single block of coordinate at each iteration. However, it is still expensive because the gradient evaluation at each iteration depends on all the data points. Further, doubly stochastic gradient methods are proposed in \cite{shen2017accelerated,dang2015stochastic}.  Among them, \cite{dang2015stochastic}, called stochastic randomized block coordinate descent (SRBCD), computes the partial derivative on one coordinate block with respect to a sample. However, SRBCD can only achieve a sublinear convergence rate due to the variance of stochastic gradients. Further, accelerated mini-batch randomized block coordinate descent (MRBCD) method \cite{zhao2014accelerated,wang2014randomized} was proposed to achieve a linear convergence rate by reducing the gradient variance \cite{li2022local}.

\begin{table*}[t]
\renewcommand\arraystretch{1.1}
	\center
\caption{Comparison between existing methods and our ADSGD method.   ``Stochasticity'' represents whether the method is stochastic on samples or coordinates, ``Sample'' represents whether the method is  scalable \wrt sample size and ``Feature'' represents whether the method is scalable \wrt feature  size.}
    \setlength{\tabcolsep}{1.2mm}
	\begin{tabular}{c|c|c|c|c|c}
		\hline
	{\textbf{Method}}  &  {\textbf{Stochasticity}} & {\textbf{Sample}} & {\textbf{Feature}} & {\textbf{Model Identification}}  &  {\textbf{Identification Rate}}  \\
		\hline	
	PGD  \cite{lions1979splitting,xiao2014proximal} &  \cmark  &\cmark &  \xmark  &  Implicit & $-$ \\
		\hline
Screening \cite{fercoq2015mind, bao2020fast}  &  \xmark  &   \xmark  & \cmark  &   Explicit & $-$    \\	
		\hline
		 ADSGD (Ours)   & \cmark  & \cmark & \cmark &   Implicit $\&$  Explicit & $O(\log (1/\epsilon_j))$ \\
		\hline
	\end{tabular}
	\label{table:method}
\end{table*}

However, all the existing methods still suffer from huge computational costs and memory usage in the practical high-dimensional scenario. The reason is that these methods require the traversal of all the features and correspondingly the overall complexity increases linearly with feature size. The pleasant surprise is that the non-smooth regularizer usually promotes the model sparsity and thus the solution of Problem (\ref{eq:general}) has only a few non-zero coefficients, aka the support set. In high-dimensional problems, model identification, aka support set identification, is a key property that can be used to speed up the optimization. If we can find these non-zero features in advance, Problem (\ref{eq:general}) can be easily solved by restricted to the support set with a significant computational gain without any loss of accuracy. 

Model identification can be achieved in two ways: implicit and explicit identification. In terms of the implicit model identification, PGD can identify the support set after a finite number of iterations $T$ \cite{lewis2002active,lewis2016proximal,liang2017activity}, which means, for some finite $K > 0$, we have $\supp(x^k) = \supp(x^*)$ holds for any $k > K$. Other popular variants of PGD \cite{liang2017activity,poon2018local} also have such an implicit identification property. However, these results can only show the existence of $T$ and cannot give any useful estimate without the knowledge of $\theta^*$ \cite{lewis2002active,lewis2016proximal}, which makes it implicit and impossible to explicitly discard the coefficients that must be zero at the optimum in advance. In terms of the explicit model identification, screening can identify the zero coefficients and discard these features directly \cite{fercoq2015mind, bao2020fast, fan2021safe, Bao2022DoublySA, liang2022variable}. By achieving dimension reduction, the algorithm only needs to solve a sub-problem and save much useless computation. A similar technique in deep neural networks is pruning \cite{han2015deep, he2017channel, wu2020intermittent, wu2020enabling, wu2022fairprune}. However, first, existing works  \cite{fercoq2015mind,ndiaye2017gap,rakotomamonjy2019screening, bao2020fast, Bao2022DoublySA}  mainly focus on the existence of identification but fail to show how fast we can achieve explicit model identification. Besides, existing work \cite{ndiaye2020screening} on explicit model identification is limited to the deterministic setting. Therefore, accelerating the model training by achieving fast explicit model identification is promising and sorely needed for high-dimensional problems in the stochastic setting.

To address the challenges above, in this paper, we propose a novel Accelerated Doubly Stochastic Gradient Descent (ADSGD) method for sparsity regularized loss minimization problems, which can significantly improve the training efficiency without any loss of accuracy. On the one hand, ADSGD computes the partial derivative on one coordinate block with respect to a mini-batch of samples to simultaneously enjoy the stochasticity on both samples and features and thus achieve 
a low per-iteration cost. On the other hand, ADSGD not only enjoys implicit model identification of the proximal gradient method, but also enjoys  explicit model identification by eliminating the inactive features. Specifically, ADSGD has two loops. We first eliminate the inactive blocks at the main loop. Within the inner loop, we only estimate the gradient over a selected active block with a mini-batch of samples. To reduce the gradient variance, we adjust the estimated gradient with the exact gradient over the selected block. Theoretically, by addressing the difficulty of the uncertainty of double stochasticity, we establish the analysis for the convergence of ADSGD. Moreover, by linking the sub-optimality gap and duality gap, we provide theoretical analysis for fast explicit model identification. Finally, based on the results of the convergence and explicit model identification  ability, we establish the theoretical analysis of the overall complexity of ADSGD. Empirical results show that ADSGD can achieve a significant computational gain than existing methods. Both the theoretical and the experimental results confirm the superiority of our method.  Table \ref{table:method} summarizes the advantages of our ADSGD over existing methods. 

\noindent{\textbf{Contributions.}} We summarize the main contributions of this paper as follows:
\begin{itemize}[leftmargin=0.25in]
\item We propose a novel accelerated doubly stochastic gradient descent method for generalized sparsity regularized problems with lower overall complexity and faster explicit model identification rate. To the best of our knowledge, this is the first work of doubly stochastic gradient method to achieve a linear model identification rate.  
\item We derive rigorous theoretical analysis for our ADSGD method for both strongly and nonstrongly convex functions. For strongly convex function, ADSGD can achieve a linear convergence rate $O( \log (1/\epsilon))$ and reduce the per-iteration cost from $O(d)$ to $O(s)$ where $s \ll d$, which improves existing methods with a lower overall complexity $O(s(n+T/\mu) \log (1/\epsilon))$. For nonstrongly convex function, ADSGD can also achieve a lower overall complexity $O(s (n+T/\epsilon) \log (1/\epsilon))$.
\item We provide the theoretical guarantee of the iteration number $T$ to achieve explicit model identification. We rigorously prove our ADSGD algorithm can achieve the explicit model identification at a linear rate $O(\log (1/\epsilon_j))$.
\end{itemize}

\section{Preliminary}

\subsection{Notations and Background}
For norm $\Omega(\cdot)$, $\Omega^D(\cdot)$ is the dual norm and defined as $\Omega^{D}(u)=\max_{\Omega(z) \leq 1}\langle z, u\rangle$ for any $u \in \mathbb{R}^{d}$ if $z \in \mathbb{R}^{d}$. Denote $\theta$ as the dual solution and $\Delta_{A}$ as the feasible space of $\theta$, the dual formulation ${\D}(\theta)$ of Problem (\ref{eq:general}) can be written as:
\begin{eqnarray} 
\underset{\theta \in \Delta_{A}}{ \max } {\D}(\theta):= -\frac{1}{n}\sum_{i=1}^{n} f_{i}^{*}(-\theta_{i}).
\label{eq:dual}
\end{eqnarray}
The dual ${\D}(\theta)$ is strongly concave for Lipschitz gradient continuous $\F(x)$ (see Proposition $3.2$ in \cite{johnson2015blitz}).

For Problem (\ref{eq:general}), based on the subdifferential \cite{kruger2003frechet,mordukhovich2006frechet}, Fermat's conditions (see \cite{bauschke2011convex} for Proposition $26.1$) holds as:
\begin{eqnarray} 
\frac{1}{n} A_j^\top \theta^* \in \lambda \partial \Omega_j(x^*_{\G_j}).
\label{eq:fermatcondition}
\end{eqnarray}
The optimality conditions of Problem (\ref{eq:general}) can be read from  (\ref{eq:fermatcondition}) as:
\begin{eqnarray} 
\frac{1}{n}\Omega^D_j(A_j^\top \theta^*) = \lambda, \quad \emph{if}  \; x^{*}_{\G_j} \neq 0; \\  
\qquad \frac{1}{n}\Omega^D_j(A_j^\top \theta^*) \leq \lambda, \quad \emph{if}  \; x^{*}_{\G_j} = 0. 
\label{eq:optimality}
\end{eqnarray}

\subsection{Definitions and Assumptions}

One key property of our method is to achieve explicit model identification, which means we can explicitly find the Equicorrelation Set in Definition 1 of the solution. 
\begin{definition} \label{definition:equicorrelationset} (\textbf{Equicorrelation Set} (see \cite{tibshirani2013lasso}))
Suppose $\theta^{*}$ is the dual optimal, the equicorrelation set is defined as 
\begin{eqnarray} 
\B^{*}:=\{j \in \{1,2,\ldots,q\}: \frac{1}{n} \Omega_{j}^{D}(A_j^{\top} \theta^{*})=  \lambda \}.
\label{eq:equicorrelationset}
\end{eqnarray}
\end{definition}  

\begin{assumption} \label{assumption:lipschitzblock}
Given the partition $\{\G_{1}, \ldots, \G_{q}\}$, all $\nabla_{\G_{j}} f_{i}(x) =  [\nabla f_{i}(x)]_{\G_{j}}$ are block-wise Lipschitz continuous with constant $L_i$, which means that for any  $x$ and $x^{\prime}$, there exists a constant $L = \max_i L_i$, we have 
\begin{eqnarray} 
\|\nabla_{\G_{j}} f_{i}(x)-\nabla_{\G_{j}} f_{i}(x^{\prime})\| \leq L\|x_{\G_{j}}-x_{\G_{j}}^{\prime}\|.
\label{eq:lipschitzblock}
\end{eqnarray}
\end{assumption}

\begin{assumption} \label{assumption:convex}
 $\F(x)$ and $\Omega(x)$ are proper, convex and lower-semicontinuous.
\end{assumption}

Assumptions \ref{assumption:lipschitzblock} and \ref{assumption:convex} are commonly used in the convergence analysis of the RBCD method \cite{zhao2014accelerated,wang2014randomized}, which are standard and satisfied by many regularized loss minimization problems. Assumption \ref{assumption:lipschitzblock} implies that there exists a constant $T \leq qL$, for any  $x$ and $x^{\prime}$, we have
\begin{eqnarray} 
\|\nabla f_{i}(x)-\nabla f_{i}(x^{\prime})\| \leq T \|x-x^{\prime}\|,
\label{eq:lipschitz}
\end{eqnarray}
i.e., $\nabla f_i(x)$ is Lipschitz continuous with constant $T$.

\section{Proposed Method}

In this section, we will introduce the accelerated doubly stochastic gradient descent (ADSGD) method with discussions.

To achieve the model identification, a naive implementation of the ADSGD method is summarized in Algorithm \ref{algorithm:ASGD}. ASGD can reduce the size of the original optimization problem and the variables during the training process. Thus, the latter problem has a smaller size and fewer variables for training, which is a sub-problem of the problem from the previous step, but generates the same optimal solution. 

We denote the original problem as $\PP_{0}$ and subsequent sub-problem as $\PP_{k-1}$ at the $k$-th iteration in the main loop of ASGD. Moreover, we define the active set at the $k$-th iteration of the main loop as $\B_{k-1}$. Thus, in the main loop of ASGD, we compute $\theta_{k-1}$ 
with the active set $\B_{k-1}$ from the previous iteration as
\begin{eqnarray} 
\theta_{k-1}= \frac{- \nabla \F( \tilde{x}_{k-1})} { \max (1, \Omega^{D}(A_{\B_{k-1}}^{\top} \nabla \F(\tilde{x}_{k-1}))/\lambda )}.
\label{eq:scaling}
\end{eqnarray}
Then we compute the intermediate duality gap $\Gap(\tilde{x}_{k-1}, \theta_{k-1}):= \PP_{k-1}(\tilde{x}_{k-1})-{\D_{k-1}}(\theta_{k-1})$ for the screening test. In step 7, we obtain new active set $\B_k$ from $\B_{k-1}$ by the screening conducted on all $j \in \B_{k-1}$ as
\begin{eqnarray} 
\frac{1}{n} \Omega^D_j(A_j^\top \theta_{k-1}) + \frac{1}{n} \Omega^D_j(A_j) r^{k-1} < \lambda  
\Rightarrow \tilde{x}^{*}_{\G_j} = 0,
\label{eq:screen_Ak}
\end{eqnarray}
where the safe region is chosen as spheres $\mathcal{R}=\mathcal{B}(\theta_{k-1}, r^{k-1})$  \cite{bao2020fast, ndiaye2017gap}. 

With obtained active set $\B_k$, we update the design matrix $A$ and the related parameter variables $x^{0},\tilde{x}$ in step 8. 
In the inner loop, we conduct all the operations on  $\B_k$. To make the algorithm scale well with the sample size, we only randomly sample a mini-batch $\I \in \{1,2,\ldots,n\}$ of samples at each iteration to evaluate the gradients.

\begin{property} \label{property:safety}
Let $\hat{x}_{\G_j}$ be the $j$-th block of $\hat{x}$ in Algorithm \ref{algorithm:ASGD}, $\forall j \in \{1,2,\ldots,q\}$, $\hat{x}_{\G_j}$ discarded by ASGD is guaranteed to be 0 at the optimum. 
\end{property}

\begin{remark}
Property \ref{property:safety} shows that ASGD is guaranteed to be safe not only for the current iteration but also for the whole training process. The safety of the screening is the foundation of the analysis of convergence and explicit model identification rate in the following part. 
\end{remark}

\begin{remark}
Property \ref{property:safety} also shows that discarding inactive variables can either decrease or make no changes to the objective function. 
\end{remark}

\begin{algorithm}[ht]
\renewcommand{\algorithmicrequire}{\textbf{Input:}}
\renewcommand{\algorithmicensure}{\textbf{Output:}}
\caption{The ASGD method}
\begin{algorithmic}[1]
\REQUIRE $\hat{x}_{0}$.
\FOR{$k=1,2,\ldots$}
\STATE $\tilde{x}_{k-1} = \hat{x}_{k-1}$.
\STATE $x^{0}_{k-1} = \tilde{x}_{k-1}$.
\STATE Compute $\theta_{k-1}$ by (\ref{eq:scaling}).
\STATE $r^{k-1} = \sqrt{2T \Gap(\tilde{x}_{k-1}, \theta_{k-1})}$.
\STATE Update $\B_k \subset \B_{k-1}$ by (\ref{eq:screen_Ak}).
\STATE Update $A_{\B_k}, x_{k}^{0}, \tilde{x}_{k}$ with $\B_k$.
\FOR{$t=1,2,\ldots,m q_k/q$}
\STATE Randomly pick $\I \subset \{1,2,\ldots,n\}$.
\STATE $x^{t}_{k} = {\prox}_{\eta,{\lambda}}(x_k^{t-1}-\eta  \nabla \F_{\I}(x_{k}^{t-1}))$.
\ENDFOR
\STATE $\hat{x}_{k} = \frac{1}{m_k} \sum_{t=1}^{m_k}x^{t}_{k}$.
\ENDFOR
\ENSURE Coefficient $\hat{x}_{k}$.
\end{algorithmic}
\label{algorithm:ASGD}
\end{algorithm}

\begin{algorithm}[ht]
\renewcommand{\algorithmicrequire}{\textbf{Input:}}
\renewcommand{\algorithmicensure}{\textbf{Output:}}
\caption{The ADSGD method}
\begin{algorithmic}[1]
\REQUIRE $\hat{x}_{0}$.
\FOR{$k=1,2,\ldots$}
\STATE $\tilde{x}_{k-1} = \hat{x}_{k-1}$.
\STATE $\tilde{\mu}_{k-1} = \nabla \F(\tilde{x}_{k-1})$.
\STATE $x^{0}_{k-1} = \tilde{x}_{k-1}$.
\STATE Compute $\theta_{k-1}$ by (\ref{eq:scaling}).
\STATE $r^{k-1} = \sqrt{2T \Gap(\tilde{x}_{k-1}, \theta_{k-1})}$.
\STATE Update $\B_k \subset \B_{k-1}$ by (\ref{eq:screen_Ak}).
\STATE Update $A_{\B_k},x_{k}^{0},\tilde{x}_{k}, \tilde{\mu}_{k}$ with $\B_k$.
\FOR{$t=1,2,\ldots,m q_k/q$}
\STATE Randomly pick $\I \subset \{1,2,\ldots,n\}$.
\STATE Randomly pick $j$ from $\B_k$.
\STATE $\mu_{k} = \nabla_{\G_j} \F_{\I}(x_{k}^{t-1}) - \nabla_{\G_j} \F_{\I}(\tilde{x}_{k}) + \tilde{\mu}_{\G_j,k}  $.
\STATE $x^{t}_{k,\G_j} = {\prox}^j_{\eta,{\lambda}}(x^{t-1}_{k, \G_j}-\eta \mu_{k})$.
\ENDFOR
\STATE $\hat{x}_{k} = \frac{1}{m_k} \sum_{t=1}^{m_k}x^{t}_{k}$
\ENDFOR
\ENSURE Coefficient $\hat{x}_{k}$.
\end{algorithmic}
\label{algorithm:ADSGD}
\end{algorithm}

\noindent{\textbf{Doubly Stochastic Gradient Update.}}
Since ASGD is only singly stochastic on samples, the gradient evaluation is still expensive because it depends on all the coordinates at each iteration. Thus, we randomly select a coordinate block  $j$ from  $\B_k$ to enjoy the stochasticity on features. Specifically, ADSGD only computes the partial derivative $\nabla_{\G_j} \F_{\I}(x_{k}^{t-1})$ on one coordinate block with respect to a sample each time, which yields a much lower per-iteration computational cost. The proximal step is computed as:
\begin{eqnarray} \label{eq:proximal}
{\prox}^j_{\eta, \lambda}(x'_{\G_j}) =  \argmin_{x_{\G_j}} \frac{1}{2\eta}\|x'_{\G_j}-x_{\G_j}\|^2 + \lambda \Omega_j(x_{\G_j}).
\end{eqnarray}
Therefore, ADSGD is doubly stochastic and can scale well with both the sample size and feature size.

\noindent{\textbf{Variance Reduction on the Selected Blocks.}} However, the gradient variance introduced by stochastic sampling  does not converge to zero. Hence, a decreasing step size is required to ensure the convergence. In that case, even for strongly convex functions, we can only obtain a sublinear convergence rate. Thanks to the full gradient we computed in step 3, we can adjust the partial gradient estimation over the selected block $\G_j$ to reduce the gradient variance with almost no additional computational costs as:
\begin{eqnarray}
\mu_{k} = \nabla_{\G_j} \F_{\I}(x_{k}^{t-1}) - \nabla_{\G_j} \F_{\I}(\tilde{x}_{k}) + \tilde{\mu}_{\G_j,k}.
\end{eqnarray}
It can guarantee that the variance of stochastic gradients asymptotically goes to zero. Thus, a constant step size can be used to achieve a linear convergence rate if $\PP$ is strongly convex.

The algorithmic framework of the ADSGD method is presented in Algorithm \ref{algorithm:ADSGD}.  ADSGD  can identify inactive blocks  and thus only active blocks are updated in the inner loop, which can make more progress for training rather than conduct useless updates for inactive blocks. Thus, fewer inner loops are required for each main loop and correspondingly huge computational time is saved.

Suppose we have $q_{k}$ active blocks for the inner loop at the $k$-th iteration, we only do $m_k = m q_k/q$ iterations in the inner loop for the current outer iteration, which means the number of the inner loops continues decreasing in our algorithm. The output for each iteration is the average result of the inner loops. As the algorithm converges, the duality gap converges to zero and thus (\ref{eq:screen_Ak}) can eliminate more inactive blocks and thus save the computational cost to a large extent.

Remarkably, the key step of the screening in Algorithm \ref{algorithm:ADSGD} is the computation of the duality gap, which only imposes extra $O(q_k)$ costs to the original algorithm at the $k$-th iteration. Note the extra complexity of screening is much less than $O(d)$ in practice, which would not affect the complexity analysis of the algorithm. Further, at the $k$-th iteration, our Algorithm \ref{algorithm:ADSGD} only requires the computation over $\B_k$, which could be much smaller than the original model over the full parameter set. Hence, in high-dimensional regularized problems, the computation costs are promising to be effectively reduced with the constantly decreasing active set $\B_k$.

On the one hand, ADSGD enjoys the implicit model identification of the proximal gradient method. On the other hand, discarding the inactive variables can further speed up the
identification rate. Thus, ADSGD is promising to achieve a fast identification rate by simultaneously enjoying the implicit and explicit model identification to yield a lower per-iteration cost.

\section{Theoretical Analysis}
We first give several useful lemmas and then provide theoretical analysis on the convergence, explicit model identification, and overall complexity.

\subsection{Useful Lemmas}

\begin{lemma}  \label{lemma:supp1}
Define $v_{\I} = \frac{1}{|\I|}\sum_{i\in \I}(\nabla f_i(x^{t-1}) - \nabla f_i(\tilde{x})) +  \nabla \F(\tilde{x})$,  $\overline{x}=\prox_{\eta}(x^{t-1}-\eta \nabla \F(x^{t-1}))$,   $\overline{x}_{\I}= \prox_{\eta}(x^{t-1}-\eta  v_{\I})$ and 
$
{\prox}_\eta(x') =  \argmin_{x} \frac{1}{2\eta}\|x'-x\|^2 + \Omega(x),
$
we have 
\begin{eqnarray} 
\E_{\I}(v_{\I}-\nabla \F(x^{t-1}))^\top(x^*-\overline{x}_{\I}) 
\leq \eta \E_{\I}\|v_{\I}-\nabla \F(x^{t-1})\|^{2}.
\end{eqnarray}
\end{lemma}

\begin{proof}
Introducing $\overline{x}$, we have
\begin{eqnarray} 
&&\E_{\I}(v_{\I}-\nabla \F(x^{t-1}))^\top(x^*-\overline{x}_{\I}) \nonumber \\
&=&\E_{\I}[(v_{\I}-\nabla \F(x^{t-1}))^\top(x^*-\overline{x}) +(v_{\I}-\nabla \F(x^{t-1}))^\top(\overline{x}-\overline{x}_{\I})] \\
&=&\E_{\I}[(v_{\I}-\nabla \F(x^{t-1}))^\top(\overline{x}-\overline{x}_{\I})]. 
\end{eqnarray}
where the second equality is obtained from $\E_{\I}[v_{\I}-\nabla \F(x^{t-1})] = 0$. 

Thus, we have
\begin{eqnarray} 
&&\E_{\I}(v_{\I}-\nabla \F(x^{t-1}))^\top(x^*-\overline{x}_{\I}) \nonumber \\
& \leq& \E_{\I}\|v_{\I}-\nabla \F(x^{t-1})\| \cdot\|{\prox}_{\eta}(x^{t-1}-\eta \nabla \F(x^{t-1}))-{\prox}_{\eta}(x^{t-1}-\eta v_{\I})\| \\
& \leq& \E_{\I}\|v_{\I}-\nabla \F(x^{t-1})\| \cdot\|(x^{t-1}-\eta \nabla \F(x^{t-1}))-(x^{t-1}-\eta v_{\I})\| \\
&=&\eta \E_{\I}\|v_{\I}-\nabla \F(x^{t-1})\|^{2}.
\end{eqnarray}
where the first inequality comes from Cauchy-Schwarz inequality and the second inequality comes from the non-expansiveness of the proximal operator, which completes the proof.
\end{proof}

\begin{lemma} \label{lemma:svrg} (See \cite{xiao2014proximal}) 
Define $v_{i} = \nabla f_i(x^{t-1}) - \nabla f_i(\tilde{x}) +  \nabla \F(\tilde{x})$, conditioning on $x^{t-1}$, we have
$
\E_{i}v_{i} =  \nabla \F(x^{t-1}), 
$
and 
\begin{eqnarray} 
\E_{i} \|v_{i}-\nabla \F(x^{t-1})\|_2^2 
\leq 4 T(\PP(x^{t-1})-\PP(x^*)+\PP(\tilde{x})-\PP(x^*)). \nonumber
\end{eqnarray}
\end{lemma}

\begin{lemma} \label{lemma:asrbcd} (See \cite{zhao2014accelerated})
Define $\delta = (\bar{x} - x)/ \eta$ and $\delta_{\G_j} = (\bar x_{\G_j} - x)/ \eta$ where $\bar x = {\prox}_{\eta}(x-\eta v)$ and 
$
{\prox}_\eta(x) = \argmin_{x'} \frac{1}{2\eta}\|x'-x\|^2 + \Omega(x'),
$
we have 
$
\E_j \delta_{\G_j} = \delta/q $ and $ \E_j \|\delta_{\G_j}\|^2 = \|\delta\|^2 /q.
$
Moreover, taking $\eta \leq 1/L$, we have 
\begin{eqnarray} 
&& \E_j [(x - x^*)^\top \delta_{\G_j} + \frac{\eta}{2} \|\delta_{\G_j}\|^2] \nonumber \\ &\leq & \frac{1}{q}\PP(x^*) + \frac{q-1}{q}\PP(x) - \E_j \PP(\bar{ x}_{\G_j}) 
+  \frac{1}{q}(v-\nabla \F(x))^\top (x^* - \bar x),
\end{eqnarray}
where $x^*= \argmin_{x} \PP(x)$.
\end{lemma}

\subsection{Strongly Convex Functions}
We establish the theoretical analysis of ADSGD for strongly convex $\F$ here. 

\subsubsection{Convergence Results}
\begin{lemma} \label{lemma:objdecrease1}
Suppose $\hat{x}_{k}$ and $\tilde{x}_{k}$ are generated from the $k$-th iteration of the main loop in Algorithm \ref{algorithm:ADSGD} and let $|\I| \geq T/L$ and $\eta<\frac{1}{4 L}$, we have 
\begin{eqnarray}  
\E \PP_k(\hat{x}_{k}) - \PP_k(x_{k}^*) \leq
\rho_k [\PP_k(\tilde{x}_{k}) - \PP_k(x_{k}^*)],
\label{eq:objdecrease1}
\end{eqnarray}
where $\rho_k = \frac{q_k}{\mu\eta(1-4L\eta)m_k} + \frac{4L\eta(m_k+1)}{(1-4L\eta)m_k}$. We can choose $|\I| = T/L$, $\eta = \frac{1}{16 L}$, and $m = 65q L/\mu$ to make $\rho_k < 2/3$.
\end{lemma} 

\begin{proof}
At the $k$-th iteration of the main loop, all the updates in the inner loop are conducted on sub-problem $\PP_{k}$ with the active set $\B_k$. At the $t$-th iteration of the inner loop for the $k$-th main loop, we randomly sample a mini-batch $\I$ and $\G_j \subseteq \B_k$. Define $v_{i} = \nabla f_i(x_{k}^{t-1}) - \nabla f_i(\tilde{x}_{k}) + \nabla \F(\tilde{x}_{k})$, based on Lemma \ref{lemma:svrg}, conditioning on $x_{k}^{t-1}$, we have
$
\E_{i}v_{i} =  \nabla \F(x_{k}^{t-1})
$
and 
\begin{eqnarray} 
\E_{i} \|v_{i}-\nabla \F(x_{k}^{t-1})\|_2^2 
\leq 4 T(\PP(x_{k}^{t-1})-\PP(x_{k}^*)+\PP(\tilde{x}_{k})-\PP(x_{k}^*)).
\end{eqnarray}
Define $\delta = (\bar x_{k} - x_{k})/ \eta$ and $\delta_{\G_j} = (\bar x_{\G_j,k} - x_{k})/ \eta$ where $\bar x_{k} = {\prox}_{\eta}(x_{k}-\eta v)$, based on Lemma \ref{lemma:asrbcd}, we have 
$
\E_j \delta_{\G_j} = \delta/q_{k} $ and  $ \E_j \|\delta_{\G_j}\|^2 = \|\delta\|^2 /q_{k}.
$
Moreover, taking $\eta \leq 1/L$, we have 
\begin{eqnarray} 
&& \E_j [(x_{k} - x_{k}^*)^\top \delta_{\G_j} + \frac{\eta}{2} \|\delta_{\G_j}\|^2] \nonumber \\ &\leq& \frac{1}{q_{k}}\PP_k(x_{k}^*)  + \frac{q_{k}-1}{q_{k}}\PP_k(x_{k}) - \E_j \PP_k(\bar x_{{\G_j},k})  + \frac{1}{q_{k}}(v-\nabla \F(x_{k}))^\top (x_{k}^* - \bar x_{k}).
\end{eqnarray}
Define $\delta_{\I,\G_j}=(x_{k}^{t} - x_{k}^{t-1})/ \eta$ and $\bar{x}_{k,\I}={\prox}_{\eta}(x_{k}^{t-1}-\eta v_{\I})$, we have 
\begin{eqnarray} 
&& \E_{\I, j}\|x_{k}^{t}-x_{k}^*\|^{2} - \|x_{k}^{t-1}-x_{k}^*\|^{2} \nonumber \\
&=&\E_{\I, j}\|x_{k}^{t-1}+\eta \delta_{\I,\G_j}-x_{k}^*\|^{2} - \|x_{k}^{t-1}-x_{k}^*\|^{2} \\
&=&2 \eta(x_{k}^{t-1}-x_{k}^*)^\top \E_{\I, j}[\delta_{\I,\G_j}]+\eta^{2} \E_{\I, j}\|\delta_{\I, \G_{j}}\|^{2} \\
&=&\E_{\I}[2 \eta(x_{k}^{t-1}-x_{k}^*)^\top \E_{j}[\delta_{\I, \G_{j}}]+\eta^{2} \E_{j}\|\delta_{\I, \G_{j}}\|^{2}] \\
&\leq& \frac{2 \eta}{q_{k}} \E_{\I}[(v_{\I}-\nabla \F(x^{t-1}_{k}))^\top(x_{k}^*-\bar{x}_{k,\I})] \nonumber
 \\ && + 2 \eta \E_{\I}[\frac{1}{q_{k}} \PP_k(x_{k}^*)  +\frac{q_{k}-1}{q_{k}} \PP_k(x^{t-1}_{k})-\E_{j} \PP_k(\bar{x}_{\G_{j},k,\I})] \\
&\leq&-2 \eta \E_{\I}[\E_{j} \PP_k(\bar{x}_{\G_{j},k,\I})-\PP_k(x_{k}^*)  -\frac{q_{k}-1}{q_{k}}(\PP_k(x_{k}^*)-\PP_k(x^{t-1}_{k}))] \nonumber \\ 
&&+\frac{8\eta^{2} T}{q_{k}|\I|}(\PP_k(x_{k}^{t-1})-\PP_k(x_{k}^*)+\PP_k(\tilde{x}_{k})-\PP_k(x_{k}^*))\\
&\leq&-2 \eta[\E_{\I, j} \PP_k(\bar{x}_{\G_{j},k,{\I}})-\PP_k(x_{k}^*)]
 +2 \eta \frac{q_{k}-1}{q_{k}}(\PP_k(x_{k}^{t-1})-\PP_k(x_{k}^*)) \nonumber \\
&&+\frac{8T \eta^{2}}{q_{k}|\I|}(\PP_k(x_{k}^{t-1})-\PP_k(x_{k}^*)+\PP_k(\tilde{x}_{k})-\PP_k(x_{k}^*)),
\label{eq:47}
\end{eqnarray}
where the first inequality is obtained by Lemma \ref{lemma:asrbcd}, and
the second inequality is obtained by  Lemma \ref{lemma:supp1}, Lemma \ref{lemma:svrg} and the fact that $\E_{\I} v_{\I}$ is the unbiased estimator of $ \nabla \F(x_{k}^{t-1})$ and
\begin{eqnarray} 
\E_{\I} \|v_{\I}-\nabla \F(x_{k}^{t-1})\|_2^2 = \frac{1}{|\I|} \E_{i}\|v_{i}-\nabla \F(x_{k}^{t-1})\|_2^2.
\end{eqnarray}
Note $x_{k}^{0}=\tilde{x}_{k}$ and $\hat{x}_{k}=\frac{1}{m_{k}} \sum_{t=1}^{m_{k}} x_{k}^{t}$, then we have 
\begin{eqnarray} 
&&\E\|x_{k}^{m_{k}}-x_{k}^*\|^{2}-\|x_{k}^{0}-x_{k}^*\|^{2}+2 \eta \sum_{t=1}^{m_{k}}(\E \PP_k(x_{k}^{t})-\PP_k(x_{k}^*)) \nonumber \\
&\leq& \frac{8 T \eta^{2} /|\I|+2 \eta(q_{k}-1)}{q_{k}} \sum_{t=1}^{m_{k}-1}(\E \PP_k(x_{k}^{t})-\PP_k(x_{k}^*))  \\ && +\frac{8 T \eta^{2}(m_{k}+1)}{q_{k}|\I|}(\PP_k(\tilde{x}_{k})-\PP_k(x_{k}^*))\\
&\leq& \frac{8 T \eta^{2} /|\I|+2 \eta(q_{k}-1)}{q_{k}} \sum_{t=1}^{m_{k}}(\E \PP_k(x_{k}^{t})-\PP_k(x_{k}^*))  \\ && +\frac{8 T \eta^{2}(m_{k}+1)}{q_{k}|\I|}(\PP_k(\tilde{x}_{k})-\PP_k(x_{k}^*)),
\label{eq:50}
\end{eqnarray}
where the first inequality is obtained by  (\ref{eq:47}). Rearranging (\ref{eq:50}), note $x_{k}^{0}=\tilde{x}_{k}$, we have 
\begin{eqnarray} 
&&2 \eta\left( \frac{1-4 \eta T /|\I|}{q_{k}}\right) \sum_{t=1}^{m_{k}}[\E \PP_k(x_{k}^{t})-\PP_k(x_{k}^*)] \nonumber \\
& \leq&\|x_{k}^{0}-x_{k}^*\|^{2}+\frac{8 T \eta^{2}(m_{k}+1)}{q_{k}|\I|}(\PP_k(\tilde{x}_{k})-\PP_k(x_{k}^*))  \\
& \leq& \frac{2}{\mu}(\PP_k(\tilde{x}_{k})-\PP_k(x_{k}^*))+\frac{8 T \eta^{2}(m_{k}+1)}{q_{k}|\I|}(\PP_k(\tilde{x}_{k})-\PP_k(x_{k}^*)).
\label{eq:52}
\end{eqnarray}
where the second inequality is obtained by the strong convexity of $\PP$. Based on the convexity of $\PP$, we have $\PP_k(\hat{x}_{k}) \leq \frac{1}{m_{k}} \sum_{t=1}^{m_{k}} \PP_k(x_{k}^{t})$. Thus, by (\ref{eq:52}), we can obtain
\begin{eqnarray} 
 && 2 \eta\left(\frac{1-4 \eta T /|\I|}{q_{k}}\right) m_{k}\left[\E \PP_k\left(\hat{x}_{k}\right)-\PP_k(x_{k}^*)\right]   \\ &\leq& \left(\frac{2}{\mu}+\frac{8T \eta^{2}(m_{k}+1)}{q_{k}|\I|}\right)\left[\PP_k\left(\tilde{x}_{k}\right)-\PP_k(x_{k}^*)\right].
\end{eqnarray}
Defining
$
\rho_k=\left(\frac{q_{k}}{\mu \eta\left(1-4 \eta T /|\I|\right) m_{k}}+\frac{4 \eta T /|\I|(m_{k}+1)}{\left(1-4 \eta T /|\I|\right) m_{k}}\right),
$
we have 
\begin{eqnarray} 
\E \PP_k\left(\hat{x}_{k}\right)-\PP_k(x_{k}^*) \leq \rho_k\left[\PP_k\left(\tilde{x}_{k}\right)-\PP_k(x_{k}^*)\right].
\end{eqnarray}
Choosing $|\I|=T / L$, we have 
\begin{eqnarray} 
 \E \PP_k\left(\hat{x}_{k}\right)-\PP_k(x_{k}^*)  \leq \left(\frac{q_{k}}{\mu\eta(1-4L\eta)m_{k}} + \frac{4L\eta(m_{k}+1)}{(1-4L\eta)m_{k}}\right) \left[\PP_k\left(\tilde{x}_{k}\right)-\PP_k(x_{k}^*)\right].
\end{eqnarray}
Considering $m_k = m q_k/q$, we can choose $\eta = \frac{1}{16 L}$, and $m = 65q L/\mu$  to make $\rho_k < 2/3$, which completes the proof.
\end{proof}

\begin{remark}
Lemma \ref{lemma:objdecrease1} shows that the overall inner loop of Algorithm \ref{algorithm:ADSGD} can decrease the expected objective function with a factor $\rho_k$ at the $k$-th iteration. 
\end{remark}

\begin{theorem}  \label{theorem:converge}
Suppose $\hat{x}_{k}$ be generated from the $k$-th iteration of the main loop in Algorithm \ref{algorithm:ADSGD} and let $|\I| \geq T/L$ and $\eta<\frac{1}{4 L}$, we have 
\begin{eqnarray} 
\E \PP_k(\hat{x}_{k}) - \PP(x^*) \leq
\rho^k [\PP(\hat{x}) - \PP(x^*)].
\label{eq:objdecrease}
\end{eqnarray}
We can choose $|\I| = T/L$, $\eta = \frac{1}{16 L}$, and $m = 65q L/\mu$ to make $\rho < 2/3$.
\end{theorem}

\begin{proof}
Considering each sub-problem $\PP_k$, based on Lemma \ref{lemma:objdecrease1}, for each main loop, we have
\begin{eqnarray} 
\E \PP_k\left(\hat{x}_{k}\right)-\PP_k(x_{k}^*) \leq \rho_k\left[\PP_k\left(\tilde{x}_{k}\right)-\PP_k(x_{k}^*)\right].
\end{eqnarray}
Since the eliminating in the ADSGD algorithm is safe, the optimal solution of all the sub-problems are the same. Thus, we have   
\begin{eqnarray} 
\PP_k(x_{k}^*) = \PP_{k-1}(x_{k-1}^*).
\label{eq:60}
\end{eqnarray}
Then, the coefficients eliminated at the $k$-th iteration of the main loop must be zeroes at the optimal, which means the eliminating stage at the $k$-th iteration of the main loop actually minimizes the sub-problem $\PP_{k-1}$ over the eliminated variables. Thus, we have 
$
\PP_k(\tilde{x}_{k}) \leq \PP_{k-1}(\tilde{x}_{k-1}). 
$
Moreover, considering $\PP_{k-1}(\hat{x}_{k-1}) = \PP_{k-1}(\tilde{x}_{k-1})$, we have 
$
\PP_{k}(\tilde{x}_{k}) \leq \PP_{k-1}(\hat{x}_{k-1}).
$

Combining above, we have 
\begin{eqnarray} 
\E \PP_k\left(\hat{x}_{k}\right)-\PP_k(x_{k}^*) &\leq& \rho_k \left[\PP_k\left(\tilde{x}_{k}\right)-\PP_k(x_{k}^*)\right] \\
& \leq& \rho_k \left[\PP_{k-1}(\tilde{x}_{k-1})-\PP_{k-1}(x_{k-1}^*)\right]. 
\end{eqnarray}
We can choose $|\I|=T / L$, $\eta = \frac{1}{16 L}$, and $m = 65q L/\mu$  to make $ \forall k, \rho_k < 2/3$. Thus,  $\exists \rho < 2/3 $, applying the above inequality recursively, we obtain
$
\E \PP_k\left(\hat{x}_{k}\right)-\PP_k(x_{k}^*) \leq 
 \rho^k \left[\PP_{0}(\hat{x}_{\B_{0}})-\PP_{0}(x_{\B_{0}}^*)\right]. 
$
From (\ref{eq:60}), we have 
\begin{eqnarray} 
\PP_k(x_{k}^*) = \PP_{k-1}(x_{k-1}^*) = \cdots = \PP_{0}(x_{\B_{0}}^*).
\end{eqnarray}
Note $\PP_{0}=\PP$ and $\B_{0}$ is the universe set for all the variables blocks,  we have
\begin{eqnarray} 
\E \PP_k\left(\hat{x}_{k}\right)-\PP(x^*) \leq 
 \rho^k \left[\PP(\hat{x})-\PP(x^*)\right], 
\end{eqnarray}
which completes the proof.
\end{proof}

\begin{remark}
Theorem \ref{theorem:converge} shows that ADSGD converges linearly to the optimal with the convergence rate $O(\log_{\frac{1}{\rho}}(1/\epsilon))$.
\end{remark}

\subsubsection{Explicit Model Identification}

\begin{lemma} \label{lemma:dualconverge}
$\lim _{k \rightarrow+\infty} \theta_{k}=\theta^*$  if $\lim _{k \rightarrow+\infty} \hat{x}_{k}=x^*$. 
\end{lemma} 

\begin{proof}
Define $\alpha_k = \max \left(1, \Omega^{D}\left(A_{\B_k}^{\top} \nabla \F( \tilde{x}_k)\right)/\lambda \right)$, we have
\begin{eqnarray} 
\left\|\theta_{k}-\theta^*\right\|_{2} &=& \left\|\frac{\nabla \F\left(\tilde{x}_{k}\right)}{\alpha_{k}}-\nabla \F(x^*)\right\|_{2} \\
& \leq&\left|1-\frac{1}{\alpha_{k}}\right|\left\|\nabla \F\left(\tilde{x}_{k}\right)\right\|_{2}+\left\|\nabla \F\left(\tilde{x}_{k}\right) - \nabla \F(x^*)\right\|_{2}.
\end{eqnarray}
Considering the right term, if $\lim_{k \rightarrow+\infty} \hat{x}_{k}=\hat{x}^*$, we have $\alpha_k \rightarrow \max \left(1, \Omega^{D}\left(A_{\B^{*}}^{\top} \nabla \F(x^*)\right)/\lambda \right) = 1$ and $\left\|\nabla \F\left(\tilde{x}_{k}\right) - \nabla \F(x^*)\right\|_{2} \rightarrow 0$. Thus, the right term converges to zero, which completes the proof.
\end{proof}

\begin{remark}
Lemma \ref{lemma:dualconverge} shows the convergence of the dual solution is guaranteed by the convergence of the primal solution.
\end{remark}

\begin{lemma} \label{lemma:boundgap}
$\exists u \in \partial\Omega^D(A_{\B_k}^\top \theta_{k})$, for all $s \in [0,1]$, we have
\begin{eqnarray} 
 \PP_k(\hat{x}_{k}) - \PP(x^*) \geq
s \Gap(\hat{x}_{k}, \theta_{k}) 
 - \frac{T s^2}{2}\|A_{\B_k}(u-\hat{x}_{k})\|^2. 
\label{eq:link}
\end{eqnarray}
\end{lemma} 

\begin{proof}
From the smoothness of $\F$, we have:
\begin{eqnarray}
 \F(\hat{x}_{k}) - \F(\hat{x}_{k} + s(u - \hat{x}_{k}))   \geq   - s \langle \nabla \F(\hat{x}_{k}), A_{\B_k}(u - \hat{x}_{k}) \rangle - \frac{s^2 T}{2}\|A_{\B_k}(u - \hat{x}_{k})\|^2.
\label{eq:75}
\end{eqnarray}
By the convexity of $\Omega$, we have:
\begin{eqnarray}
\Omega(\hat{x}_{k}) - \Omega(\hat{x}_{k} + s (u - \hat{x}_{k})) \geq s (\Omega(\hat{x}_{k}) - \Omega(u)).
\label{eq:76}
\end{eqnarray}
Moreover, we  have
\begin{eqnarray}
&& \Omega(\hat{x}_{k}) - \Omega(u) - \langle \nabla \F(\hat{x}_{k}), A_{\B_k}(u - \hat{x}_{k}) \rangle \nonumber \\
&=& \Omega(\hat{x}_{k}) + \Omega^D(A_{\B_k}^\top \theta_{k}) - \langle u, A_{\B_k}^\top \theta_{k} \rangle - \langle \nabla \F(\hat{x}_{k}), A_{\B_k}(u - \hat{x}_{k}) \rangle \\
&=& \Omega(\hat{x}_{k}) + \Omega^D(A_{\B_k}^\top \theta_{k}) + \langle \nabla \F(\hat{x}_{k}), A_{\B_k} \hat{x}_{k} \rangle \\
&=& \Omega(\hat{x}_{k}) + \Omega^D(A_{\B_k}^\top \theta_{k}) + \F(\hat{x}_{k}) + \F^*(-\theta_{k}) \\
&=& \Gap(\hat{x}_{k}, \theta_{k}),
\end{eqnarray}
where the first equality comes from $
\Omega(u) = \langle u, A_{\B_k}^\top \theta_{k} \rangle - \Omega^D(A_{\B_k}^\top \theta_{k})$, the third equality comes from $\F(\hat{x}_{k}) = \langle \nabla \F(\hat{x}_{k}), A_{\B_k}\hat{x}_{k} \rangle - \F^*(-\theta_{k})$.

Therefore, for any $\hat{x}_{k}$ and $u$, we have:
\begin{eqnarray}
&& \PP(\hat{x}_{k}) - \PP(x^*) \\
&\geq& \PP(\hat{x}_{k}) - \PP(\hat{x}_{k} + s(u - \hat{x}_{k})) \\
&=& [\Omega(\hat{x}_{k}) - \Omega(\hat{x}_{k} + s (u - \hat{x}_{k}))]  + [\F(\hat{x}_{k}) - \F(\hat{x}_{k}  + s(u - \hat{x}_{k}))] \\
&\geq& s [\Omega(\hat{x}_{k}) - \Omega(u) - \langle \nabla \F(\hat{x}_{k}), A_{\B_k}(u - \hat{x}_{k}) \rangle]  - \frac{s^2}{2} T \|A_{\B_k}(u-\hat{x}_{k})\|^2 \\
&= & s \Gap(\hat{x}_{k}, \theta_{k}) - \frac{s^2}{2} T \|A_{\B_k}(u-\hat{x}_{k})\|^2,
\end{eqnarray}
where the second equality comes from (\ref{eq:75}) and (\ref{eq:76}), which completes the proof.
\end{proof}

\begin{remark}
The difficulty to analyze the model identification rate of ADSGD is that the screening is conducted on the duality gap while the convergence of the algorithm is analyzed based on the sub-optimality gap. Note the sub-optimality gap at the $k$-th iteration of the main loop is computed as $\PP_k(\hat{x}_{k}) - \PP(x^*)$, Lemma \ref{lemma:boundgap} links the sub-optimality gap and the duality gap at the $k$-th iteration.
\end{remark}

\begin{theorem} \label{theorem:screenability}
Define $\Delta_j \triangleq \frac {n\lambda - \Omega^D_j(A^\top_{j} \theta^{*})}{ 2 \Omega^D_j(A_j)}$, denote $\sigma_{A}^{2}$ as the spectral norm of $A$, suppose $\Omega$ has a bounded support within a ball of radius $M$, given any $\gamma \in (0,1)$, any block that $j \notin \B^{*}$ are correctly identified by ADSGD at iteration $\log_{\frac{1}{\rho}}(1/\epsilon_j)$  with at least probability $1-\gamma$ where $\rho$ is from Theorem \ref{theorem:converge} and $\epsilon_j = \frac{1}{32}\frac{\Delta_j^4 \gamma}{T^3 \sigma_{A}^{2} M^2 (\PP(\hat{x}) - \PP(x^*))}$.
\end{theorem} 

\begin{proof}
Based on the eliminating condition, any variable block $j$ can be identified at the $k$-th iteration  when 
\begin{eqnarray} 
 \frac{1}{n} \Omega^D_j(A^\top_{j} \theta^{*}) \leq \max\limits_{\theta \in \mathcal{R}} \frac{1}{n} \Omega^D_j(A_j^\top \theta) \leq  \frac{1}{n} \Omega^D_j(A^\top_{j} \theta^{*}) + \frac{2}{n} \Omega^D_j(A_j)  r^k < \lambda.
\end{eqnarray}
Thus, we have that variable block $j$ can be identified  when
$
r^k <\frac {\lambda -  \frac{1}{n}\Omega^D_j(A^\top_{j} \theta^{*})}{ \frac{2}{n} \Omega^D_j(A_j)}.
$
Denote $\Delta_j \triangleq \frac {\lambda -  \frac{1}{n}\Omega^D_j(A^\top_{j} \theta^{*})}{ \frac{2}{n} \Omega^D_j(A_j)}$, considering $\tilde{x}_{k}=\hat{x}_{k}$ and $r^k = \sqrt{2T \Gap(\tilde{x}_{k}, \theta_{k})}$, we have that any variable block $j$ can be identified at the $k$-th iteration when
\begin{eqnarray} 
 \Gap(\hat{x}_{k}, \theta_{k}) < \frac{\Delta_j^2}{2T}.
\label{eq:82}
\end{eqnarray}

For $u_k \in \partial\Omega^D(A_{\B_k}^\top \theta_{k})$, considering $\Omega$ has a bounded support within a ball of radius $M$, we have
\begin{eqnarray}
\|A(u_k - \hat{x}_{k})\| \leq 2 \sigma_{A} M.
\end{eqnarray}
Based on Lemma \ref{lemma:boundgap}, we have
\begin{eqnarray}
\Gap(\hat{x}_{k}, \theta_{k}) &\leq \frac{1}{s} (\PP_k(\hat{x}_{k}) - \PP(x^*)) + 2T \sigma_{A}^{2} M^2 s.
\end{eqnarray}
Minimizing the right term over $s$, we have
\begin{eqnarray}
\Gap(\hat{x}_{k}, \theta_{k}) \leq \sqrt{8 T \sigma_{A}^{2} M^2 (\PP_k(\hat{x}_{k}) - \PP(x^*))}.
\end{eqnarray}
Thus, we can make
\begin{eqnarray}
\sqrt{8 T \sigma_{A}^{2} M^2 (\PP_k(\hat{x}_{k}) - \PP(x^*))} \leq \frac{\Delta_j^2}{2T}.
\label{eq:91}
\end{eqnarray}
to ensure the eliminating condition (\ref{eq:82}) holds. Since (\ref{eq:91}) can be reformulated as 
\begin{eqnarray}
 \PP_k(\hat{x}_{k}) - \PP(x^*) \leq \frac{1}{32}\frac{\Delta_j^4}{T^3 \sigma_{A}^{2} M^2},
\label{eq:95}
\end{eqnarray}
if $\exists k \in \mathbb{N}^{+}$, we have (\ref{eq:95}) hold, we can ensure the eliminating condition hold at the $k$-iteration.   

From Theorem \ref{theorem:converge}, we have
\begin{eqnarray} 
\E \PP_k(\hat{x}_{k}) - \PP(x^*) \leq
\rho^k [\PP(\hat{x}) - \PP(x^*)].
\label{eq:92}
\end{eqnarray}
Thus, if we let
$
 (\PP(\hat{x}) - \PP(x^*))\rho^k \leq \frac{1}{32}\frac{\Delta_j^4}{T^3 \sigma_{A}^{2} M^2},
$
we have 
\begin{eqnarray} 
\E \PP_k(\hat{x}_{k}) - \PP(x^*) \leq
\frac{1}{32}\frac{\Delta_j^4}{T^3 \sigma_{A}^{2} M^2}.
\end{eqnarray}
By Markov inequality, we have  
\begin{eqnarray} 
 \Prob\left(\PP_k(\hat{x}_{k}) - \PP(x^*)\geq \frac{1}{32}\frac{\Delta_j^4}{T^3 \sigma_{A}^{2} M^2} \right)   &\leq& \frac{32 T^3 \sigma_{A}^{2} M^2}{\Delta_j^4}
(\E \PP_k(\hat{x}_{k}) - \PP(x^*))  \\  &\leq&
\frac{32 T^3 \sigma_{A}^{2} M^2}{\Delta_j^4} \rho^k (\PP(\hat{x}_{k}) - \PP(x^*)). 
\end{eqnarray}
Denote $\epsilon_j = \frac{1}{32}\frac{  \Delta_j^4 \gamma}{T^3 \sigma_{A}^{2} M^2 (\PP(\hat{x}) - \PP(x^*))}$,  if we choose
$ k \geq \log_{\frac{1}{\rho}}(\frac{1}{\epsilon_j}) $, we have
\begin{eqnarray} 
 \Prob\left(\PP_k(\hat{x}_{k}) - \PP(x^*)\geq \frac{1}{32}\frac{\Delta_j^4}{T^3 \sigma_{A}^{2} M^2} \right)     &\leq&
\frac{32 T^3 \sigma_{A}^{2} M^2}{\Delta_j^4} \epsilon_j (\PP(\hat{x}_{k}) - \PP(x^*)) 
\\  &=& \gamma.
\end{eqnarray}
Thus, for $ k \geq \log_{\frac{1}{\rho}}(\frac{1}{\epsilon_j}) $, we have
\begin{eqnarray} 
\PP_k(\hat{x}_{k}) - \PP(x^*) \leq
\frac{1}{32}\frac{\Delta_j^4}{T^3 \sigma_{A}^{2} M^2},
\end{eqnarray}
with at least probability $1-\gamma$, which means any variable that $j \notin \B^{*}$ are correctly detected and successfully eliminated by the ADSGD algorithm at iteration $\log_{\frac{1}{\rho}}(\frac{1}{\epsilon_j})$  with at least probability $1-\gamma$, which completes the proof.
\end{proof}

\begin{remark}
Theorem \ref{theorem:screenability} shows that the equicorrelation set $\B^{*}$ can be identified by  ADSGD at a linear rate $O(\log_{\frac{1}{\rho}}(1/\epsilon_j))$ with at least probability $1-\gamma$. We use the Lipschitzing trick in \cite{dunner2016primal} to restrict the function $\Omega$ within a bounded support.
\end{remark}

\subsubsection{Overall Complexity}
\begin{corollary} \label{corollary:complexity1}
Suppose the size of the active features in set $\B_k$ is $d_{k}$ and $d^*$ is the size of the active features in $\B^*$, given any $\gamma \in (0,1)$, let $K_m = O(\log_{\frac{1}{\rho}}(1/\epsilon_j))$, we have $d_k$ is decreasing and $d_{K_m} $ equals to $ d^*$ with at least probability $1-\gamma$.  Define $s = \frac{1}{K_c}\sum_{k=1}^{K_c} d_k $ where  $K_c=O(\log_{\frac{1}{\rho}}(1/\epsilon))$, 
the overall complexity of ADSGD is $O((n+T/\mu)s \log(1/\epsilon))$.
\end{corollary} 

\begin{proof}
The first part of Corollary \ref{corollary:complexity1} is the direct result of  Theorem \ref{theorem:screenability}. For the second part, Theorem \ref{theorem:converge}
shows that the ADSGD method converges to the optimal with the convergence rate $O(\log\frac{1}{\epsilon})$. For each main loop, the algorithm runs $m_k$ inner loops. Thus, since $m = 65q L/\mu$ and $m_k = m q_k/q$, the ADSGD algorithm takes $O((1+\frac{q_k L}{\mu})\log\frac{1}{\epsilon})$ iterations to achieve $\epsilon$ error.

For the computational complexity, considering the $k$-th iteration of the main loop, the algorithm is solving the sub-problem $\PP_k$ and the complexity of the outer loop is $O(n d_k)$. Within the inner loop, the complexity of each iteration is $d_k|\I|/q_k$. Thus, the complexity of the inner loop is $m_k d_k|\I|/q_k$. Let $|\I|=T/L$, define $s = \frac{1}{K_c}\sum_{k=1}^{K_c} d_k $ where  $K_c=O(\log_{\frac{1}{\rho}}(\frac{1}{\epsilon}))$, the overall complexity for the ADSGD algorithm is $O((n+\frac{T}{\mu})s \log\frac{1}{\epsilon})$, which completes the proof.
\end{proof}

\begin{remark}
Please note the difference between the number of features as $d_k$ and the number of blocks as $q_k$ at the $k$-th iteration. In the high-dimensional setting, we have $d^*\ll d$ and $s\ll d$. Thus, Corollary \ref{corollary:complexity1} shows that ADSGD can simultaneously achieve linear convergence rate and low per-iteration cost, which improves ProxSVRG and MRBCD with the overall complexity $O((n+T/\mu)d \log(1/\epsilon))$ at a large extent in practice. 
\end{remark}

\subsection{Nonstrongly Convex Functions}
For nonstrongly convex $\F$, we can use a perturbation approach for analysis. Suppose $x^0$ is the initial input and $\mu_p$ is a positive parameter, adding a perturbation term $\mu_p \|x-x^0\|^2$ to Problem (\ref{eq:general}), we have:
\begin{eqnarray}  \label{eq:general2}
\min\limits_{x \in \mathbb{R}^{d} }   \F(x) +\mu_p \|x-x^0\|^2 + \Omega(x).
\end{eqnarray}
If we solve (\ref{eq:general2}) with ADSGD, since we can treat $\F(x) +\mu_p \|x-x^0\|^2$ as the data-fitting loss, we know the loss is $\mu_p$-strongly convex, we can obtain the convergence result for  (16) as $O((n+T/\mu_p)s \log(1/\epsilon))$. Suppose $\hat{x}_{k}$ be generated from the $k$-th iteration of the main loop in Algorithm \ref{algorithm:ADSGD} where $k=O((n+T/\mu_p)s \log\frac{2}{\epsilon})$ and $\vec{x}^*$ is the optimum solution of  (\ref{eq:general2}), let $|\I| \geq T/L$ and $\eta<\frac{1}{4 L}$, we obtain: 
\begin{eqnarray} 
\E \PP_k(\hat{x}_{k}) + C_p - \PP(\vec{x}^*) -\mu_p \|\vec{x}^*-x^0\|^2 \leq \epsilon/2,
\end{eqnarray}
where $C_p$ is the expectation of the perturbation term, which is always positive. Thus, we have
\begin{eqnarray} 
\E \PP_k(\hat{x}_{k})  &\leq& \epsilon/2 +  \PP(\vec{x}^*) + \mu_p \|\vec{x}^*-x^0\|^2 
\\
&\leq& \epsilon/2 +  \PP(x^*) + \mu_p \|x^*-x^0\|^2 
\end{eqnarray}
where the second inequality is obtained because $\vec{x}^*$ is the optimum solution of (\ref{eq:general2}).
If we set $\mu_p = \frac{\epsilon}{2\|x^*-x^0\|^2}$, we have $\E \PP_k(\hat{x}_{k}) - \PP(x^*) \leq \epsilon$. Since $2\|x^*-x^0\|^2$ is a constant, the overall complexity of Algorithm 1 for nonstrongly convex function is $k=O((n+T/\epsilon)s \log(1/\epsilon))$, which also improves ProxSVRG and MRBCD with the overall complexity $O((n+T/\epsilon)d \log(1/\epsilon))$ for nonstrongly convex function.

\section{Experiments}

\subsection{Experimental Setup}
\label{section:experimentalsetup}
\noindent\textbf{Design of Experiments:}
We perform extensive experiments on real-world  datasets for two popular sparsity regularized models Lasso and sparse logistic regression shown in (\ref{eq:lasso}) and (\ref{eq:logistic}) respectively to demonstrate the superiority of our ADSGD \wrt the efficiency. 
\begin{eqnarray}
\underset{x \in \mathbb{R}^{d}}{\min} \ \frac{1}{n} \sum_{i=1}^{n} \frac{1}{2}(y_{i}-x_{i}^\top x)^{2}+\lambda\|x\|_{1},
\label{eq:lasso}
\end{eqnarray}
\begin{eqnarray}
 \underset{x \in \mathbb{R}^{d}}{\min} \ \frac{1}{n} \sum_{i=1}^{n} (-y_{i}a_i^\top x + \log(1+\exp(a_i^\top x)))  +\lambda\|x\|_{1}.
\label{eq:logistic}
\end{eqnarray}
To validate the efficiency of ADSGD, we compare the convergence results of ADSGD w.r.t the running time with competitive algorithms ProxSVRG  \cite{xiao2014proximal} and MRBCD \cite{zhao2014accelerated,wang2014randomized} under different setups. We do not include the results of ASGD because the naive implementation is very slow.  The batch size of ADSGD is chosen as $10$ and $20$ respectively.

\begin{table}[t]
\centering
\caption{The descriptions of the datasets.}
\setlength{\tabcolsep}{5.6mm}
\begin{tabular}{lcc}
\toprule
Dataset & Samples & Features \\
\midrule

   PlantGO  & 978 & 3091  \\
        
     Protein      & 17766 & 357 \\
     
         Real-sim    & 72309 & 20958   \\

    Gisette & 6000 & 5000      \\

    Mnist    & 60000 & 780    \\
      Rcv1.binary    & 20242 & 47236    \\
  
\bottomrule
\end{tabular}
\label{table:datasets}
\end{table}

\noindent\textbf{Datasets:}
Table \ref{table:datasets} summarizes  the benchmark datasets used in our experiments.  Protein, Real-sim, Gisette, Mnist, and Rcv1.binary datasets are from the LIBSVM repository, which is available at \url{https://www.csie.ntu.edu.tw/~cjlin/libsvmtools/datasets/}. PlantGO is from \cite{xu2016multi}, which is available at \url{http://www.uco.es/kdis/mllresources/}. Note Mnist is the binary version of the original ones by classifying the first half of the classes versus the left ones.

\noindent\textbf{Implementation Details:}
All the algorithms are implemented in MATLAB. We compare the average running CPU time of different algorithms. The experiments are evaluated on a 2.30 GHz machine.  For the convergence results of Lasso and sparse logistic regression, we present the results with $\lambda_1 = \lambda_{\max}/2 $ and $\lambda_2 = \lambda_{\max}/4 $. Notably, $\lambda_{\max}$ is a parameter that, for all $\lambda \geq \lambda_{\max}$, $x^*$  must be $0$. Specifically, we have  $\lambda_{\max}=\frac{1}{n}\|A^\top y\|_\infty$ for Lasso and $\lambda_{\max}=\frac{1}{n}\|A^\top G(0)\|_\infty$ for sparse logistic regression where $G(\theta) \triangleq \frac{e^\theta}{1+e^\theta}-y$. Please note, for each setting, all the compared algorithms share the same hyperparameters for a fair comparison. We set the mini-batch size as 10 for the compared algorithms. The coordinate block number is set as $q = 10$. Other hyperparameters include the initial inner loop number $m$ and step size $\eta$, which are selected to achieve the best performance. For Lasso, we perform the experiments on PlantGO, Protein, and  Real-sim. For sparse logistic regression, we perform the experiments on Gisette, Mnist, and Rcv1.binary.

\begin{figure*}[!th]
  \centering
  	\begin{subfigure}[b]{0.32\textwidth}     
  	\centering
    \includegraphics[width=1.5in]{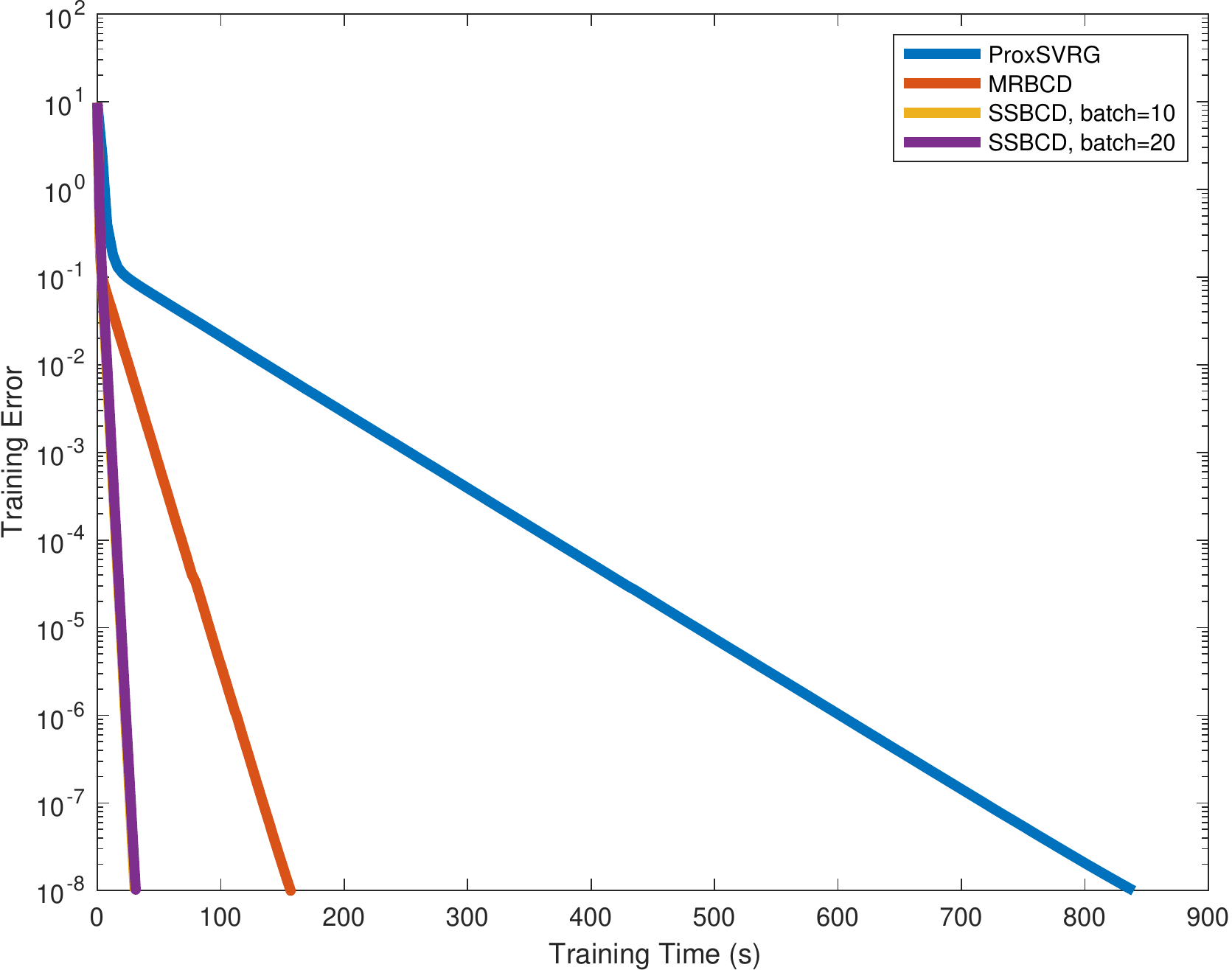}
    \caption{PlantGo}
    \end{subfigure}
    \begin{subfigure}[b]{0.32\textwidth}
    \centering
    \includegraphics[width=1.5in]{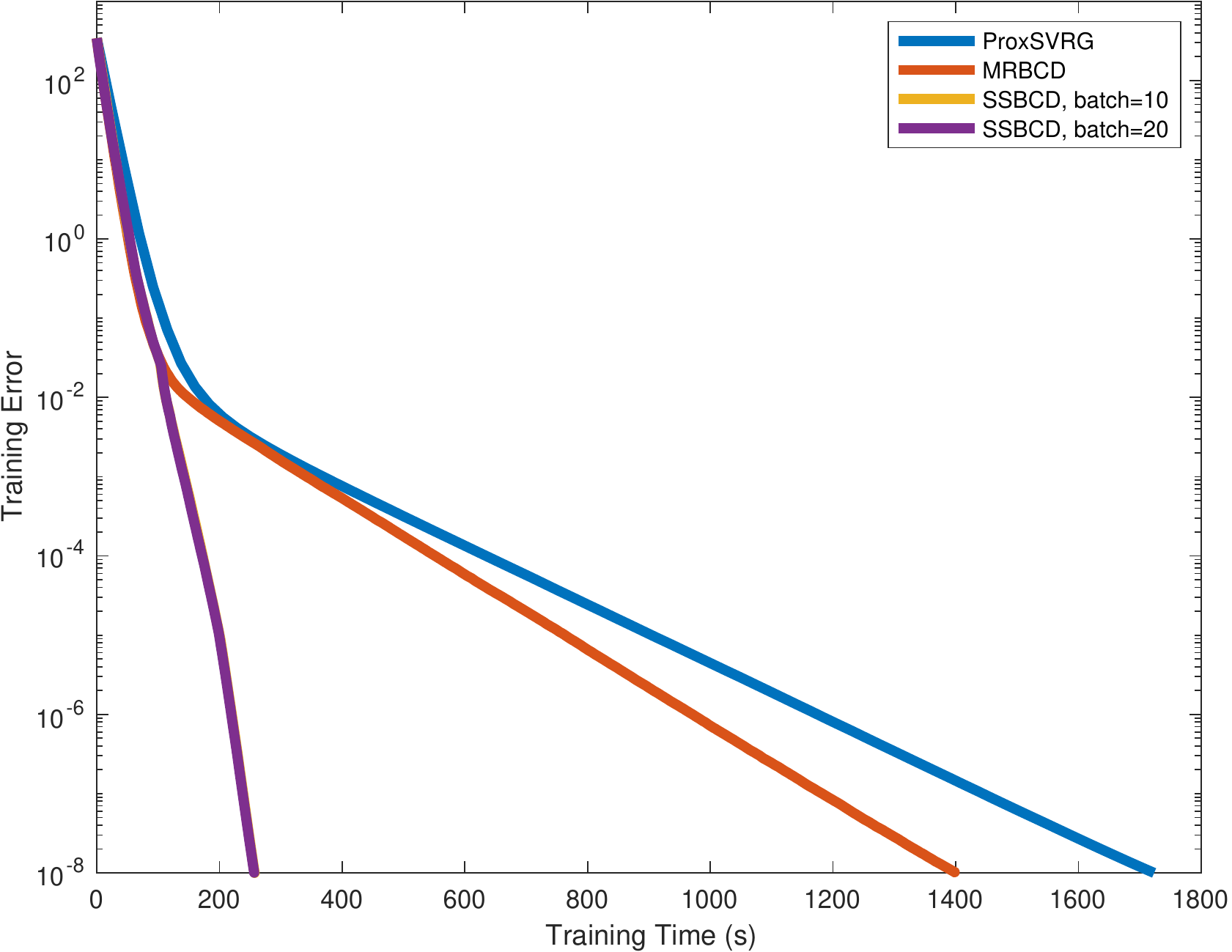}
    \caption{Protein}
    \end{subfigure}
    \begin{subfigure}[b]{0.32\textwidth}
    \centering
    \includegraphics[width=1.5in]{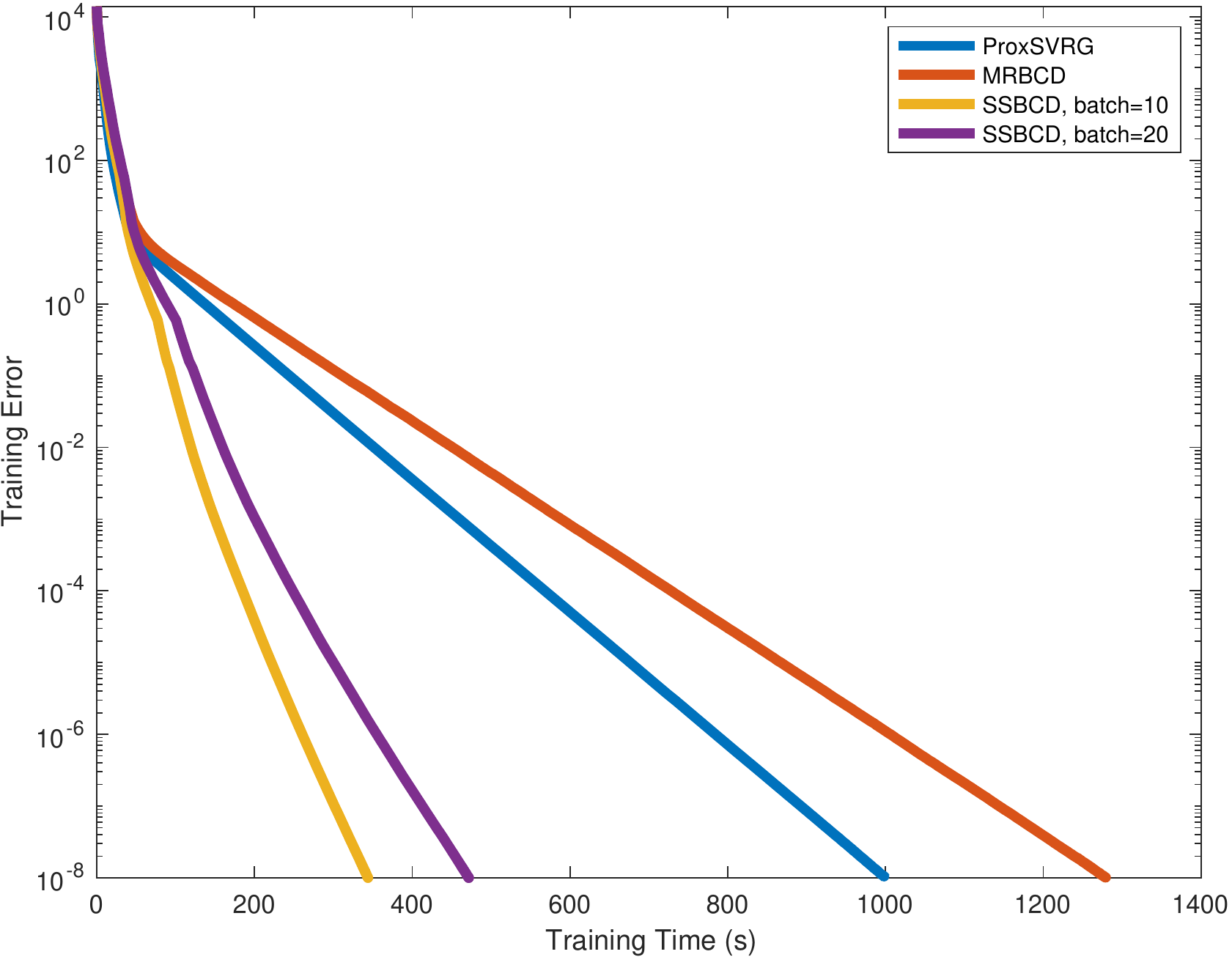}
    \caption{Real-sim}
    \end{subfigure}
  \caption{Convergence results of different algorithms for Lasso on different datasets.}
\label{fig1} 
\end{figure*}

\begin{figure*}[!th]
  \centering
  	\begin{subfigure}[b]{0.32\textwidth}     
  	\centering
    \includegraphics[width=1.5in]{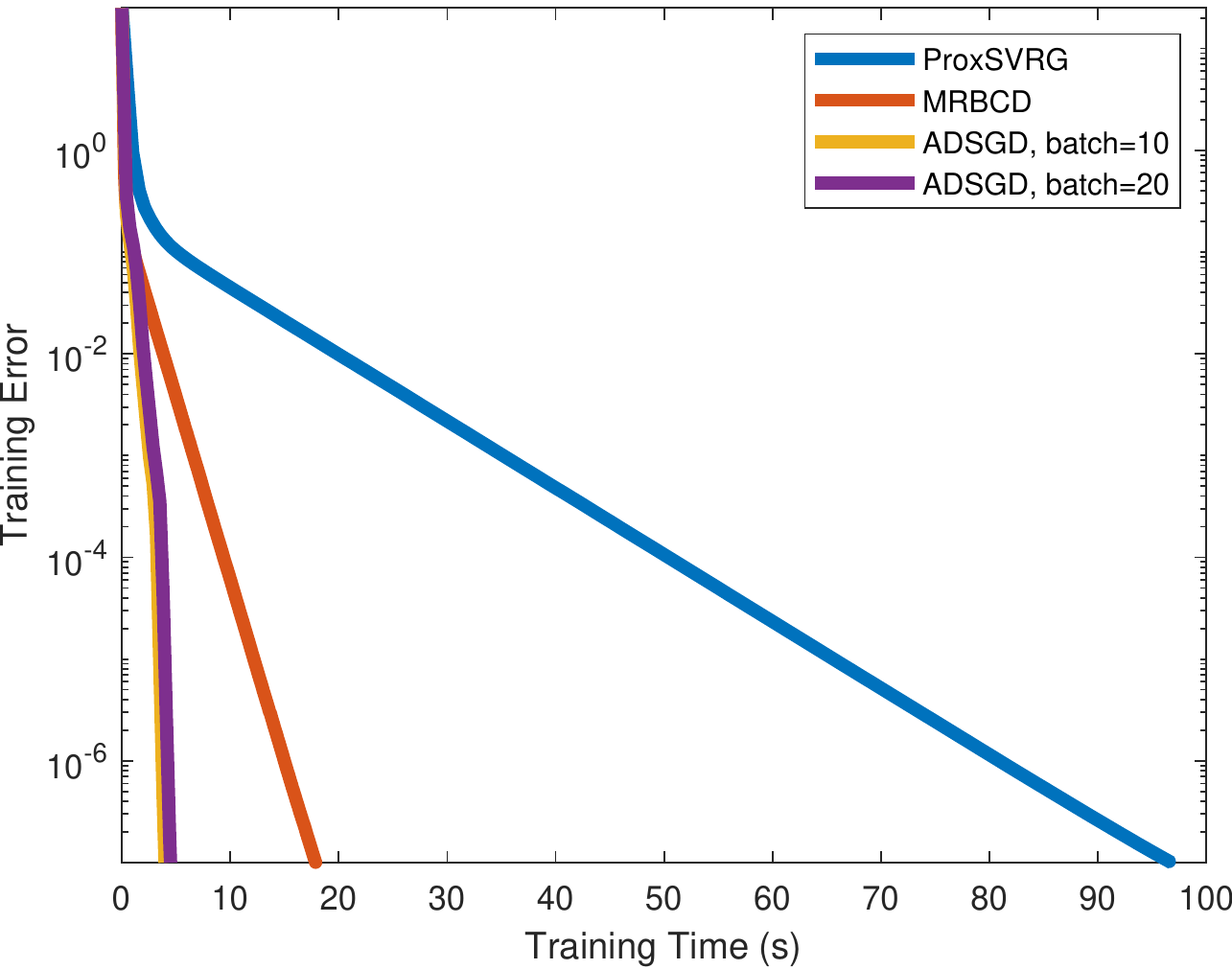}
    \caption{PlantGo}
    \end{subfigure}
    \begin{subfigure}[b]{0.32\textwidth}
    \centering
    \includegraphics[width=1.5in]{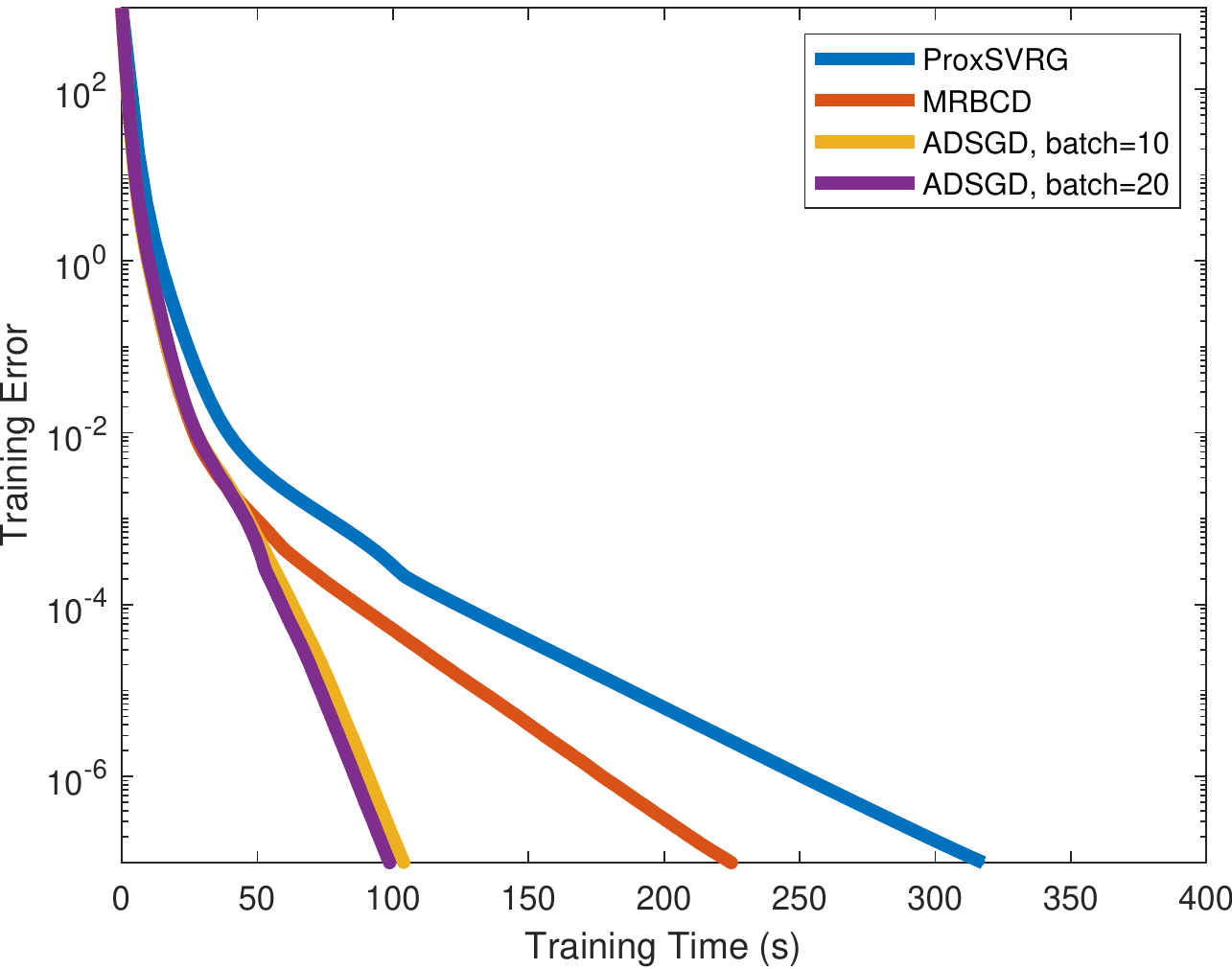}
    \caption{Protein}
    \end{subfigure}
    \begin{subfigure}[b]{0.32\textwidth}
    \centering
    \includegraphics[width=1.5in]{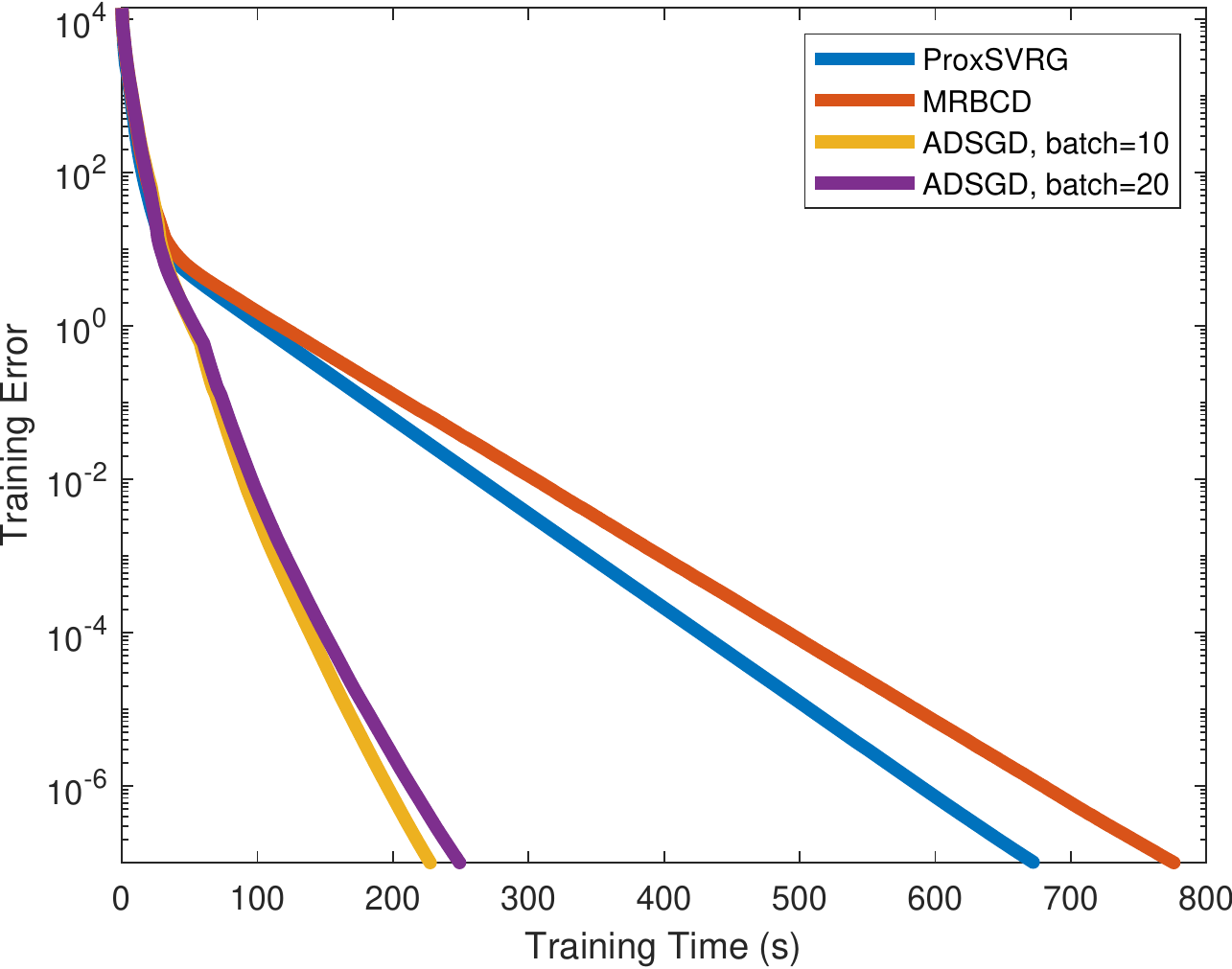}
    \caption{Real-sim}
    \end{subfigure}
  \caption{Convergence results of different algorithms for Lasso on different datasets.}
\label{fig2} 
\end{figure*}

\subsection{Experimental Results}
\paragraph{Lasso Regression}

Figures \ref{fig1}(a)-(c) provide the convergence results for Lasso on three datasets with $\lambda = \lambda_1$. Figures \ref{fig2}(a)-(c) provide the convergence results with $\lambda = \lambda_2$. The results confirm that ADSGD always converges much faster than MRBCD and ProxSVRG under different setups, even when $n  \gg d$ for Protein. This is because, as the variables are discarded, the optimization process is mainly conducted on a sub-problem with a much smaller size and thus requires fewer inner loops.  Meanwhile, the screening step imposes almost no additional costs on the algorithm. Thus,  ADSGD can achieve a lower overall complexity, compared to MRBCD and ProxSVRG conducted on the full model.

\paragraph{Sparse Logistic Regression}

Figures \ref{fig3}(a)-(c) provide the convergence results for sparse logistic regression on three datasets with $\lambda = \lambda_1$. The results also show that  ADSGD spends much less running time than MRBCD and ProxSVRG for all the datasets, even when $n  \gg d$ for Mnist. This is because our method solves the models with a smaller size and the screening step imposes almost no additional costs for the algorithm.

\begin{figure*}[!th]
 \centering
    \begin{subfigure}[b]{0.32\textwidth}
    \centering
    \includegraphics[width=1.5in]{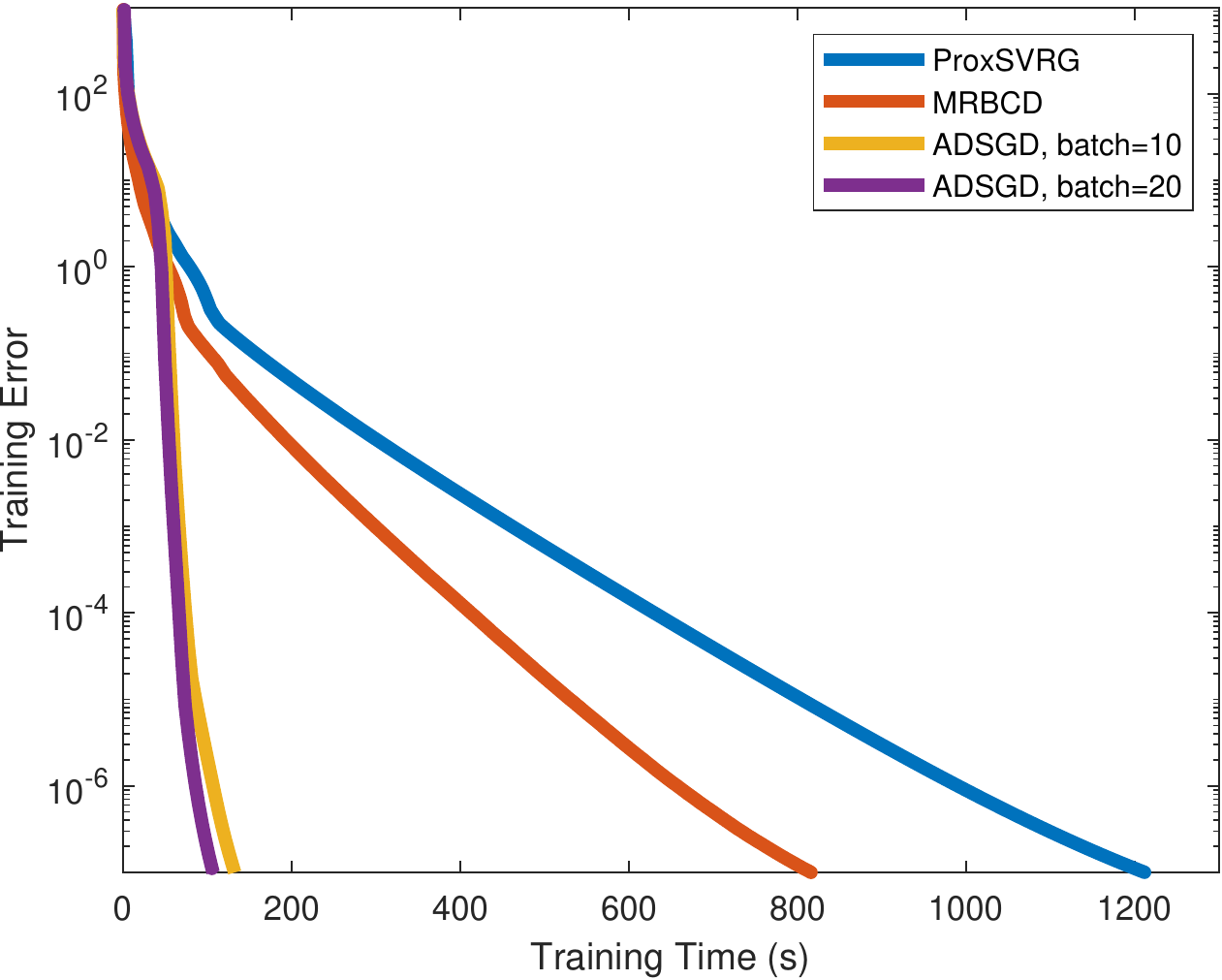}
    \caption{Gisette}
    \end{subfigure}
    \begin{subfigure}[b]{0.32\textwidth}
    \centering
    \includegraphics[width=1.5in]{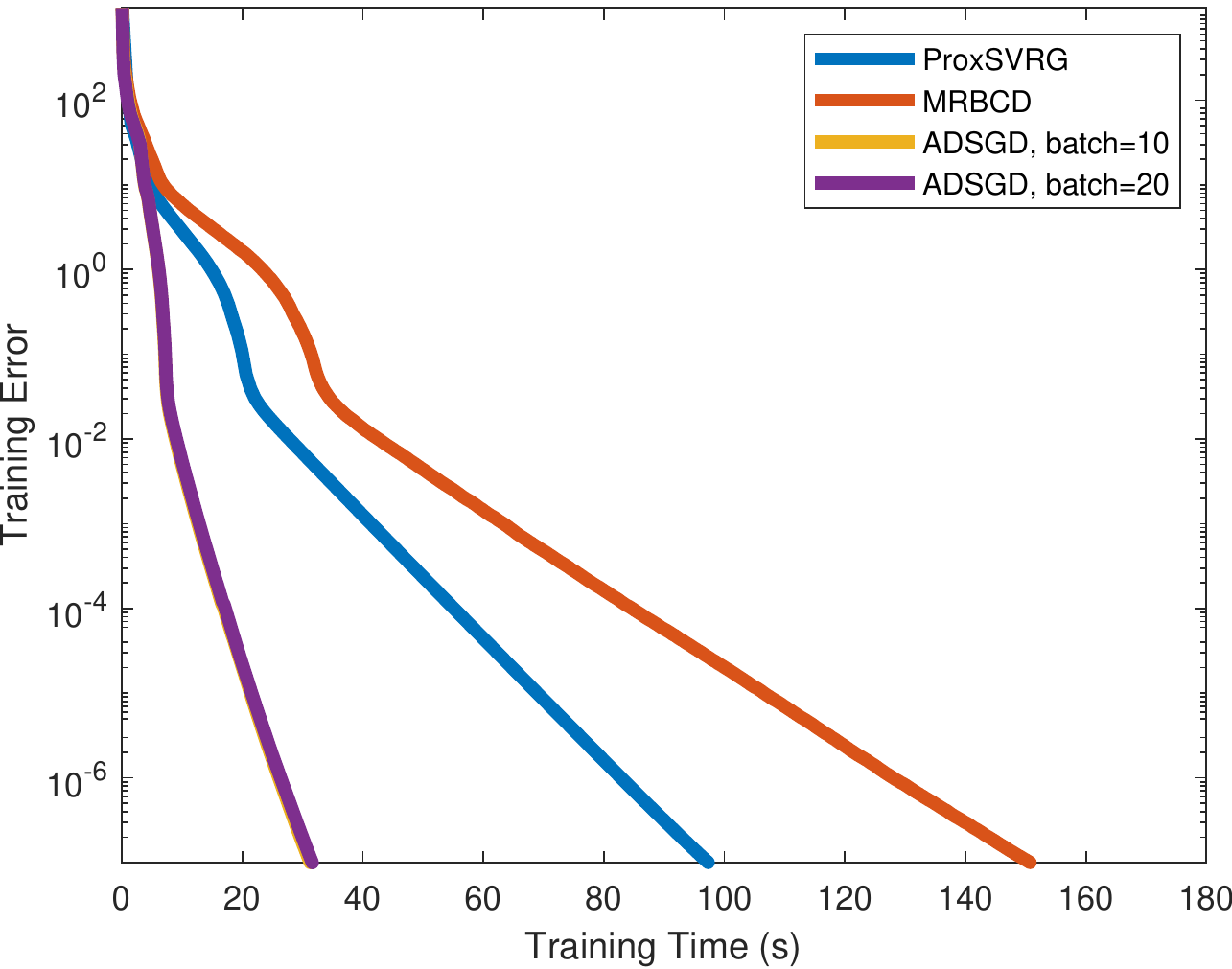}
      \caption{Mnist}
    \end{subfigure}
  	\begin{subfigure}[b]{0.32\textwidth}    
  	\centering
    \includegraphics[width=1.5in]{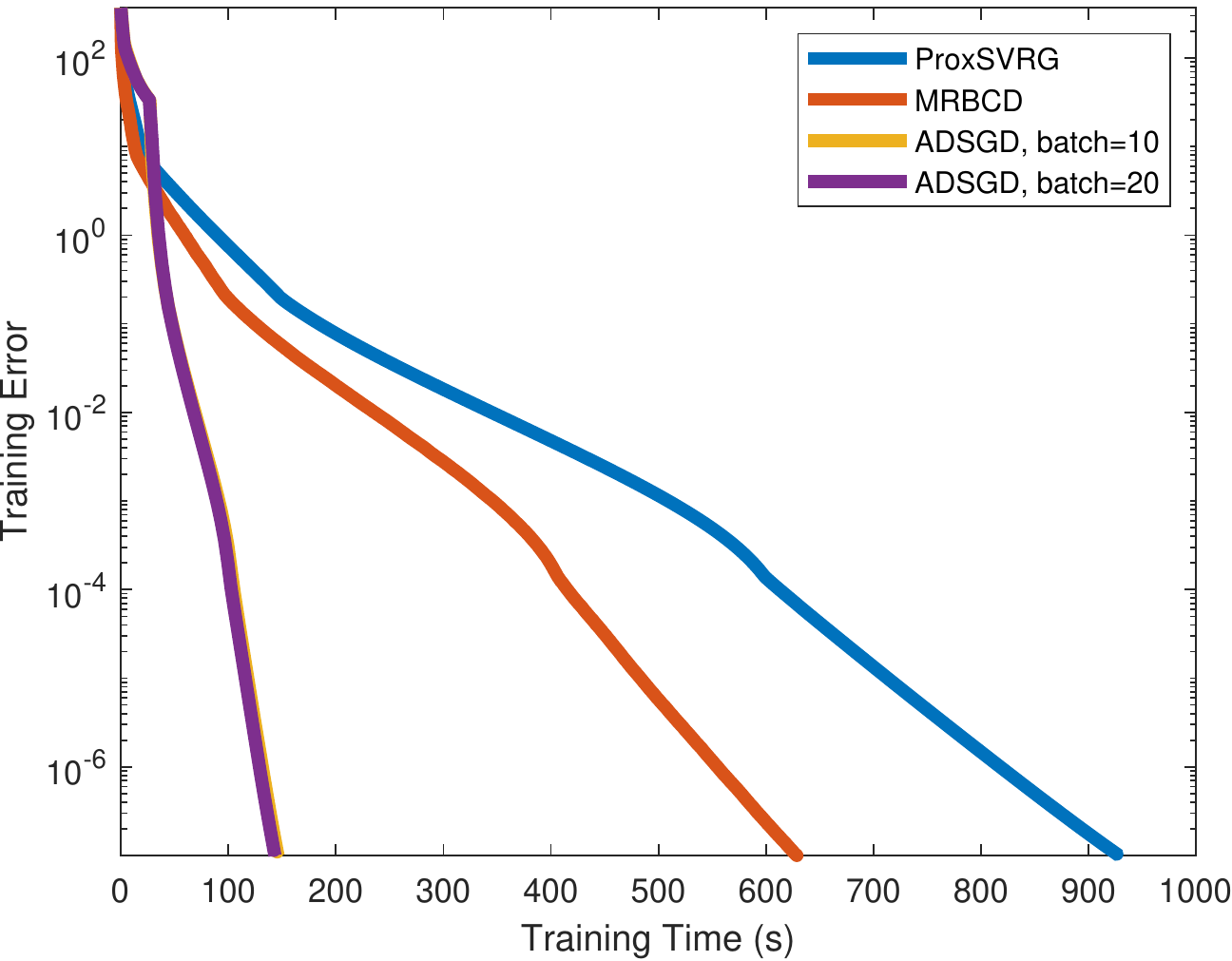}
  \caption{Rcv1.binary}
    \end{subfigure}
  \caption{Convergence results of different methods for sparse logistic regression on different datasets.}
\label{fig3} 
\end{figure*}

\section{Conclusion}
In this paper, we proposed an accelerated doubly stochastic gradient descent method for sparsity regularized minimization problem with linear predictors, which can save much useless computation by constantly identifying the inactive variables without any loss of accuracy. Theoretically, we proved that our ADSGD method can achieve lower overall computational complexity and linear rate of explicit model identification. Extensive experiments on six benchmark datasets for popular regularized models demonstrated the efficiency of our method.

\bibliographystyle{abbrv}
\bibliography{ADSGD}


\end{document}